\documentclass{article}

\PassOptionsToPackage{numbers, compress}{natbib}
\usepackage[main, final]{neurips_2026}
\usepackage[utf8]{inputenc}
\usepackage[T1]{fontenc}
\usepackage{hyperref}
\usepackage{url}
\usepackage{booktabs}
\usepackage{amsmath, amsthm, amssymb, amsfonts, mathtools, bbm}
\usepackage{nicefrac}
\usepackage{microtype}
\usepackage{graphicx}
\usepackage{subcaption}
\usepackage[table]{xcolor}
\usepackage{wrapfig}
\usepackage{algorithm}
\usepackage{algorithmic}
\usepackage{multirow}
\usepackage{enumitem}
\usepackage{placeins}

\def\keywordname{{\bfseries \emph Keywords}}
\def\keywords#1{\par\addvspace\medskipamount{\rightskip=0pt plus1cm
\def\and{\ifhmode\unskip\nobreak\fi\ $\cdot$
}\noindent\keywordname\enspace\ignorespaces#1\par}}

\newcommand{\conv}[1]{\operatorname{ConvHull}(#1)}
\newcommand{\flip}[2]{\operatorname{Flip}(#1; #2)}
\newcommand{\poly}{\mathcal{P}}
\newcommand{\vpoly}{\mathcal{V}}

\newcommand{\tri}{\triangle}
\newcommand{\trispace}[1]{\mathcal{T}(#1)}
\newcommand{\flipspace}[1]{F(#1)}
\newcommand{\simplex}{\sigma}
\newcommand{\circuit}{Z}

\newcommand{\modelvalue}{V_{\psi}}
\newcommand{\modelactor}{\pi_{\theta}}
\newcommand{\modelencoder}{\mathrm{Enc}_{\kappa}}
\newcommand{\valueparam}{\psi}
\newcommand{\actorparam}{\theta}
\newcommand{\encoderparam}{\kappa}

\usepackage{xspace}
\newcommand{\algoname}{TriSearch\xspace}

\DeclareMathOperator*{\argmin}{argmin}
\DeclareMathOperator*{\argmax}{argmax}

\DeclareMathOperator{\Lk}{Lk}

\newtheorem{theorem}{Theorem}[section]
\newtheorem{proposition}[theorem]{Proposition}

\newtheorem{definition}[theorem]{Definition}
\newtheorem{problem}[theorem]{Problem}

\title{
TriSearch: Learning to Optimize Triangulations via Bistellar Flips} 

\author{
  Yiran Wang \\
  UCLA \\
  yiranwang1027@ucla.edu \\
  \And
  Guido Mont\'ufar \\
  UCLA \& MPI MiS \\
  montufar@math.ucla.edu
}

\begin{document}
\maketitle

\begin{abstract}
We introduce TriSearch, a reinforcement-learning framework for optimizing objectives over triangulations of a polytope via bistellar flips. The key idea is a circuit-supported subtriangulation action representation: feasible flips are encoded by their supporting circuit and realized local subtriangulation, enabling a learned policy to rank them using local geometric and combinatorial features. This yields a dimension-agnostic interface and enables efficient traversal of the flip graph without explicit enumeration of the full triangulation space. Instantiated in 3D and 4D, TriSearch generalizes zero-shot from small training instances to larger polytopes with exponentially larger search spaces. It achieves top performance on metric objectives in 3D and, in 4D, discovers more distinct Fine, Regular, Star triangulations of reflexive polytopes, corresponding to Calabi-Yau threefolds, than existing samplers under a fixed budget. Our code is publicly available at \url{https://github.com/IanWangg/trisearch}.

\end{abstract}

% \keywords{Reinforcement Learning, Triangulation, Bistellar Flips, Computational Geometry}

\section{Introduction} 

A triangulation of a $d$-dimensional polytope %$\poly \subset \mathbb{R}^d$ 
is a decomposition %of $\poly$ 
into $d$-dimensional simplices that meet only along shared faces. 
Triangulations are fundamental objects in computational and combinatorial geometry. 
They underpin a wide range of downstream pipelines: in mesh generation, finite element analysis, and computer graphics, triangulations discretize a continuous geometric domain so that numerical solvers and rendering algorithms can operate over a finite set of simplices~\cite{frey2007mesh,bern1995mesh,shewchuk2002delaunay}. 
In mathematical physics, triangulations of 4D reflexive polytopes parametrize large families of Calabi-Yau threefolds~\cite{batyrev1993dual,kreuzer2000complete}. 

Different applications value different properties of a triangulation, so the practical task is %rarely to produce any triangulation but 
to find one that is \emph{good} for a chosen objective. 
Mesh generation and finite element analysis prioritize element shape and count to control solver cost and numerical accuracy~\cite{bern1995mesh,shewchuk2002delaunay}, whereas computational geometry studies combinatorial and geometric objectives such as minimizing total edge length and minimizing the number of simplices~\cite{edelsbrunner1988tetrahedrizing,below2004complexity,mulzer2008minimum}. 
Constructing a Calabi-Yau threefold requires finding a \emph{fine, regular, star} triangulation (FRST) of a 4D reflexive polytope and applying Batyrev's toric construction to map it to a corresponding threefold~\cite{batyrev1993dual}. Within this constrained family, the practical goal is to discover as many distinct topological classes as possible~\cite{kreuzer2000complete,demirtas2022cytools}. 
Each of these objectives selects a different region of the same combinatorial space, motivating a unified search framework that can target any of them.

For a $d$-dimensional polytope $\poly$, triangulation optimization is generally hard, especially for $d > 2$. 
This difficulty stems from the cardinality of the triangulation space $\trispace{\poly}$, which 
can grow exponentially in both the ambient dimension $d$ and the number of vertices $|\vpoly(\poly)|$, with a standard worst-case upper bound of $2^{O(|\vpoly(\poly)|^{\lceil d/2 \rceil})}$~\cite{dey1993counting}. 
Within this exponentially large search space, only a few special cases are known to admit polynomial-time algorithms~\cite{rajan1991optimality}. 
%Other commonly studied problems, 
In contrast, many natural objectives are computationally intractable: for example, 
%such as the decision version of minimization of number of simplices in $d>2$ and minimization of the sum of edge length in $d \geq 2$ have 
deciding whether a triangulation minimizes the number of simplices for $d>2$, or the total edge length for $d\geq2$, is NP-hard~\cite{edelsbrunner1988tetrahedrizing,below2004complexity,mulzer2008minimum}. 

The triangulation space $\trispace{\poly}$ has a natural local structure given by \emph{bistellar flips}. 
A circuit is a minimal affinely dependent set of vertices that supports two local triangulations. A bistellar flip replaces the realized local triangulation in the current triangulation with the other one. 
These moves organize %feasible 
triangulations into a flip graph $\mathcal{G}(\poly)$, whose nodes are triangulations and whose edges correspond to single flips. 
%replaces one circuit-supported local subtriangulation by the other. 
At any given state, only circuits whose realized subtriangulation is contained in the current triangulation are flippable, so the set of feasible moves is local and changes after each move. 
This state-dependent, geometry-constrained action space makes standard neural combinatorial optimization methods difficult to apply directly. 
%as they are not designed for this geometry-constrained and state-dependent action space.

This observation leads to our formulation of the triangulation problem, and the design of \algoname{}. 
We identify triangulation optimization as search on an implicit flip graph, where states are complete triangulations and actions are currently feasible bistellar flips. 
We then reformulate the problem as reinforcement learning on this graph, but do not ask the policy to learn geometric validity from data. 
Instead, we use a geometric routine~\cite{rambau2002topcom} to expose the valid local moves, and the learned policy decides which move is useful for long-range optimization. 
This formulation requires an architecture that can compare a state-dependent set of graph actions. 
We address this by representing each feasible action through the circuit supporting the flip and the realized local subtriangulation structure. 
This encodes both a certificate of admissibility and the local replacement performed by the move. 
Since circuits are defined intrinsically in any fixed dimension, the same action representation can handle different flip arities and objectives, and applies to both optimization and FRST discovery. Changing the task modifies only the reward, not the search interface.

\paragraph{Main Contributions.} 
We summarize our contributions as follows:
%\vspace{-0.3em}
\begin{itemize}[leftmargin=*]

    \item We identify triangulation optimization as reinforcement learning on an implicit flip graph, and introduce a learned search framework for this formulation. Geometric routines enumerate valid local moves, while a circuit-supported policy architecture scores which move is promising for long-range search.
    
    \item Empirically, \algoname{} shows strong performance in triangulation optimization under a fixed flip budget. 
    It achieves the lowest aggregate relative gap among evaluated search algorithms in both $3$D and $4$D. 
    Policies trained on smaller polytopes generalize to larger unseen instances, demonstrating that the learned local search generalizes beyond the hardness regime during training.
    
    \item In the Calabi-Yau (CY) setting, we show that \algoname{} has real application-level potential beyond benchmark optimization. The same framework becomes a practical FRST sampler and outperforms the existing \textsc{CYTools} samplers under the same budget by discovering more distinct CY threefold classes.
\end{itemize}
%\vspace{-0.3em}

\section{Related Work}

\textbf{Deep Learning for Combinatorial Optimization.} 
A large body of work leverages deep learning and reinforcement learning to address combinatorial optimization problems~\cite{bengio2021machine, mazyavkina2021reinforcement,cappart2023combinatorial}. 
Prior work mostly focused on routing, scheduling, or graph-optimization settings. 
In these settings, states and actions have standard representations. 
Triangulation optimization is an important problem, but it has remained under-studied by the machine learning community. 
Constructive approaches build solutions incrementally, with
representatives including pointer networks, attention-based policies, and policy-gradient models for routing~\cite{vinyals2015pointer,bello2016neural,kool2018attention,nazari2018reinforcement,kwon2020pomo}. 
This paradigm is effective when partial solutions are meaningful and feasibility can be checked or enforced locally.
Triangulation optimization does not naturally fit this setting:
once a $d$-simplex is placed inside $\poly$, the remaining region is generally non-convex in dimension above $2$, 
and deciding whether such a region admits a triangulation is itself %another 
NP-hard~\cite{ruppert1992difficulty}. 
Thus a partial list of simplices is not a reliable state for learning, and the validity of the final triangulation remains a global constraint. 

Another line of work focuses on learning local-improvement rules for complete solutions, such as routing tours, local rewrites, or Boolean assignments~\cite{wu2021improvement,chen2019neurewriter,barrett2020ecodqn}. 
These methods are designed around problem-specific actions for TSP, MaxCut, CVRP, and JSSP~\cite{wu2021improvement,chen2019neurewriter,barrett2020ecodqn,falkner2022neurols}. 
Local-search controllers such as NeuroLS~\cite{falkner2022neurols} operate at a meta level by learning acceptance criteria, operator choice, or perturbation on top of predefined local-search procedures, rather than scoring primitive moves themselves. 
In triangulation optimization, the primitive actions are bistellar flips, whose admissibility depends on the currently realized circuit subtriangulation. 
The policy therefore cannot act on a fixed set of nodes, edges, variables, or tokens. 
%We instead 
Instead, we use a geometric routine~\cite{rambau2002topcom} to enumerate valid flips, 
and train the policy to select among %those moves through 
them based on their circuit-supported local subtriangulation. 

\textbf{Triangulation Optimization and FRST Sampling.} 
Triangulations, circuits, regular triangulations, and secondary polytopes have a well-developed theory~\cite{de2010triangulations,gelfand1994discriminants}, with bistellar flips serving as the standard local moves between triangulations~\cite{lawson1972transforming,pachner1991pl}. 
Optimizing over this space is challenging, as reflected in hardness results 
for minimum-size triangulations in dimension three and for minimum-weight triangulation~\cite{below2004complexity,mulzer2008minimum}. 
%This problem is structurally different from 
Unlike many combinatorial-optimization benchmarks, this setting lacks a natural formulation for general-purpose solvers such as OR-Tools, MOSEK, or Gurobi. 
Instead, triangulation-specific software such as \textsc{TOPCOM} is used for exact traversal of the flip graph~\cite{rambau2002topcom}.
In the CY setting, Batyrev's construction and the Kreuzer-Skarke classification reduce a large computational pipeline to sampling fine, regular, and star triangulations (FRSTs) of 4D reflexive polytopes~\cite{batyrev1993dual,kreuzer2000complete, demirtas2020boundingkreuzerskarkelandscape,macfadden2026dnacalabiyauhypersurfaces,macfadden2026efficientalgorithmgeneratinghomotopy}. 
\textsc{CYTools} is a widely used computational and sampling tool for this pipeline~\cite{demirtas2022cytools}. 
Concurrent work, CYTransformer~\cite{yip2025transforming}, adopts a constructive approach, training an encoder-decoder transformer to generate candidate triangulations one token at a time while relying on \textsc{CYTools} to verify whether the output is an FRST. 
%This design gives a useful learned generator, but 
While this yields a useful learned generator, it inherits the structural difficulty noted above: generated sequences are not guaranteed to be valid triangulations or FRSTs. 
%Its reported FRST generation rate is below $50\%$ in the largest setting shown, with hodge number $h^{1,1}=10$. 
In the largest reported setting, with Hodge number $h^{1,1}=10$, the FRST generation rate remains below $50\%$. 
In contrast, \algoname does not generate full triangulations and then reject invalid outputs. 
Instead, it starts from a valid triangulation and navigates the flip graph via certified bistellar flips, %so every accepted 
ensuring that every move remains within the valid triangulation space by construction. 

\textbf{Geometric Deep Learning.} 
Geometric deep learning builds neural models for coordinate-based data, such as point clouds and meshes, while respecting underlying symmetries~\cite{bronstein2021geometric, satorras2021n,liao2022equiformer,liao2023equiformerv2}. 
Graph neural networks (GNNs) are widely used to represent instances in combinatorial optimization~\cite{cappart2023combinatorial}, but most %such models 
operate on %plain 
pairwise graphs and do not fully exploit the underlying geometric structure. 
% of the underlying problem. 
Complementary higher-order approaches, such as hypergraph and simplicial neural networks~\cite{feng2019hypergraph,ebli2020simplicial,wu2023simplicial}, model interactions beyond pairwise edges via faces and incidence relations. 
Triangulation optimization requires geometric encoding and higher-order local information, but the central issue is the action representation. 
Each action corresponds to a flippable circuit, whose realized subtriangulation carries combinatorial and geometric information. 
\algoname reflects this structure through a global equivariant graph neural network (EGNN) encoder over the current triangulation and an actor that scores circuit-supported subtriangulations.

%\section{Background and Problem Formulation}
\section{Problem Formulation and Preliminaries} 

%\subsection{Geometry Preliminaries}
\subsection{Optimal Triangulations} 

A $d$-polytope $\poly$ with %a finite 
vertex set $\vpoly(\poly) = \{p_1, \dots, p_n\} \subset \mathbb R^d$ is the convex hull of the vertices, $\poly = \conv{\{p_1, \dots, p_n\}} \subset \mathbb{R}^d$. 
We abbreviate the notation of vertex set to $\vpoly$ whenever the context is clear. 
A $d$-simplex $\simplex$ is a $d$-polytope with $(d+1)$ extreme points, e.g., a $2$-simplex is a triangle and a $3$-simplex is a tetrahedron. 
A triangulation is a collection 
%is natively defined on the vertex set, forming a structure over $\vpoly(\poly)$ to decompose $\poly$ into 
of full-dimensional simplices $\tri = \{\simplex_1, \dots, \simplex_k\}$ that decompose a polytope:  
%\footnote{A triangulation can also be defined as a \emph{simplicial complex} including lower dimensional simplices. Here we omit the lower dimensional simplices in the notation for clarity of the presentation, but one can always recover them by using the highest dimensional simplces.}  
%It is formally defined below.
\begin{definition}[Triangulation] \label{def:triangulation}
Let $\poly \subset \mathbb{R}^d$ be a $d$-polytope with vertex set $\vpoly(\poly)$.
A \emph{triangulation} $\tri = \{\simplex_1, \dots, \simplex_k\}$ of $\poly$ is a decomposition of $\poly$ into $d$-simplices with the following properties: 
\vspace{-0.7em}
\begin{itemize}
  \item (Convex Hull Union) the union of all $d$-simplices equals $\poly$: $\cup_{i=1}^k \simplex_i = \poly$.
  \vspace{-0.6em}
  \item (Intersection) the intersection of any two $d$-simplices is either empty or a shared face.
  \vspace{-0.6em}
  \item (Vertex Union) the vertices of the simplices are in $\vpoly(\poly)$: $\cup_{i=1}^k \vpoly(\simplex_i) \subseteq \vpoly(\poly)$. 
\end{itemize}
\end{definition}
A triangulation is \emph{fine} when the \emph{Vertex Union} property holds with equality, i.e., $\cup_{i=1}^k \vpoly(\simplex_i) = \vpoly(\poly)$, meaning it uses all vertices of the polytope. 
This condition is automatically met when $\vpoly(\poly)$ is the set of extreme points of the convex hull, but %it is 
non-trivial in %more general 
settings such as %in the 
FRST sampling, % problem, 
where the polytopes are \emph{reflexive polytopes} containing interior vertices. 
A triangulation $\tri$ is \emph{regular} if there exists a set of heights, such that we can lift the vertex set one dimension higher to $\mathbb{R}^{d+1}$, and projecting the lower envelope of the convex hull back to $\mathbb{R}^{d}$ recovers the simplices. 

For a fixed polytope $\poly$, let $\trispace{\poly}$ denote the space of all triangulations of $\poly$. 
%\gm{As mentioned in the introduction, the number of triangulations can be exponential in the number of vertices and doubly exponential in the dimension of the polytope.} 
%
The central optimization problem in this paper is to search this space for a triangulation that is best for a given objective. 

\begin{problem}[Optimal Triangulation] \label{def:optimal-tri}
In a fixed dimension $d$, given a polytope $\poly$ and an objective function $f: \trispace{\poly} \rightarrow \mathbb{R}$, an optimal triangulation is defined as $\tri^\star \in \argmin_{\tri \in \trispace{\poly}} f(\tri)$.
\end{problem}

\subsection{Navigating Triangulation Spaces via Flip Graphs} 

We will navigate the space of triangulations by local moves between triangulations implemented by local modifications defined as follows. 
Given a polytope $\poly$, a circuit $\circuit$ is a minimal affinely dependent subset of $\vpoly(\poly)$. 
% as defined in Definition~\ref{def:circuit}. 
%It is the local object that certifies where a triangulation can be changed. 
% \begin{definition}[Circuit] 
% \label{def:circuit}
% Let $\vpoly \subset \mathbb{R}^d$ be a vertex set.
% A subset $\circuit \subseteq \vpoly$ is a \textbf{circuit} if: (i) the vertices in $\circuit$ are affinely dependent, and (ii) every proper subset $S \subset \circuit$ is affinely independent.
% \end{definition}
%
For any circuit $\circuit$, the affine dependence has positive and negative parts by Radon partition~\cite{radon1921mengen}, which induce two unique triangulations of $\conv{\circuit}$, denoted by $\tri^+(\circuit)$ and $\tri^-(\circuit)$. 
A circuit $\circuit$ is \emph{flippable} in $\tri$ if one of these two local triangulations is realized in $\tri$. 
\begin{definition}[Flippable Circuit] \label{def:flippability}
Let $\tri$ be a triangulation of a $d$-polytope $\poly$.
A circuit $\circuit \subset \vpoly(\poly)$ is \emph{flippable} in $\tri$ if either $\tri^+(\circuit) \subseteq \tri$ or $\tri^-(\circuit) \subseteq \tri$.
\end{definition}
%In this case %Since the affine dependence is determined only up to sign, given a triangulation $\tri$ 
By convention, we denote by $\tri^+(\circuit)$ the realized local subtriangulation in $\tri$, and by $\tri^-(\circuit)$ its replacement. 
A \emph{bistellar flip} replaces the realized side by the other side as defined in Definition~\ref{def:flip}, with an illustrative example in $3$D in Figure~\ref{fig:bistellar}. 
\begin{definition}[Bistellar Flip] \label{def:flip}
Let $\circuit$ be a flippable circuit in $\tri$, with the labels oriented so that $\tri^+(\circuit) \subseteq \tri$.
The \emph{bistellar flip} on $\circuit$ replaces $\tri^+(\circuit)$ by $\tri^-(\circuit)$, producing a new triangulation $\tri' := \left( \tri \setminus \tri^+(\circuit) \right) \cup \tri^-(\circuit)$.
\end{definition}
When $\conv{\circuit}$ has dimension below $d$, the circuit gives the core of the move.
Given the current triangulation $\tri$ and a flippable circuit $\circuit$, the surrounding triangulation uniquely determines the common link $L_{\tri}(\circuit)$.
We therefore use $\circuit \in \flipspace{\tri}$ to denote the full circuit-supported flip action, with this link suppressed.
The notation $\tri^+(\circuit)$ denotes the full-dimensional local subtriangulation removed from $\tri$, and $\tri^-(\circuit)$ denotes the full-dimensional replacement.
This realized subtriangulation fixes the direction of the bistellar flip, since the same local move can be traversed in both directions.
Appendix~\ref{app:radon-circuits} unpacks the linked construction and the proof of the uniqueness.

% \begin{proposition}[Deterministic Flip Transition] 
% Let $\tri$ be a triangulation and let $\circuit \in \flipspace{\tri}$ be a flippable circuit. Given the local subtriangulation $\tri^+(\circuit)$ of $\circuit$ currently realized in $\tri$, the bistellar flip supported on $\circuit$ determines a unique successor triangulation $\flip{\tri}{\circuit}$.
% \end{proposition}

Through bistellar flips, the triangulation space becomes a flip graph. 
The flip graph $\mathcal{G}(\poly)$ is a bidirectional graph, whose node set is the set of all triangulations $\trispace{\poly}$, and the edge set is all the bistellar flips. 
For a triangulation $\tri$, we use $\flipspace{\tri}$ to denote the set of flippable circuits, each understood as its full circuit-supported flip action.
Its one-step neighborhood is $\mathcal{N}(\tri)= \{\flip{\tri}{\circuit} \mid \circuit \in \flipspace{\tri}\}$, where $\flip{\tri}{\circuit}$ is the triangulation obtained by applying the flip supported by $\circuit$. However, the flip graph does not need to be fully-connected when $d > 2$, there are known examples with disconnected flip graph in 6D~\cite{santos2005non}. 
%Our flip centric problem formulation allows \algoname{} to learn the rich structure, but does not guarantee a global optimality.

%\vspace{-0.5em}
\subsection{Triangulation Optimization via Reinforcement Learning}

Finding an exact solution for a generic objective usually requires enumerating the whole triangulation space. 
This is infeasible beyond small instances because $\trispace{\poly}$ grows exponentially with $|\vpoly(\poly)|$. 
Instead of treating this space as an unstructured set, we leverage its local graph structure, and reformulate the problem into a reinforcement learning problem. 

\begin{wrapfigure}[10]{r}{0.4\textwidth}
    \centering
    \includegraphics[width=0.9\linewidth]{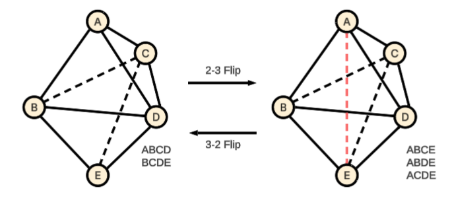}
    \caption{An illustration of bistellar flips in $3$D.}
    \label{fig:bistellar}
\end{wrapfigure}

The optimal triangulation problem can %therefore 
be viewed as search over $\mathcal{G}(\poly)$. 
Exact enumeration algorithms are reliable on small instances, but they scale poorly with the size of the flip graph.
We formulate the search over the triangulation graph as a Markov Decision Process (MDP)
$
    \mathcal{M}=\{\trispace{\poly}, \flipspace{\cdot}, R_f, H, \gamma\}.
$
The state space $\trispace{\poly}$ is the vertex set of $\mathcal{G}(\poly)$, and the action space at a state $\tri$ is the state-dependent set of flippable circuits $\flipspace{\tri}$ under the convention above.
\algoname~does not materialize $\mathcal{G}(\poly)$, so this MDP is a conceptual description of the search problem.
At each visited triangulation, geometric routines enumerate only the currently feasible outgoing flips $\flipspace{\tri}$, and the policy chooses among those locally visible actions.

The reward function $R_f$ is defined by the improvement in the objective induced by a flip. 
For a minimization objective, we use
$
    R_f(\tri, \circuit) = f(\tri) - f(\flip{\tri}{\circuit}),
    \ \circuit \in \flipspace{\tri}
$,
, the sign is reversed for maximization case. 
The horizon $H$ specifies the number of flips %allowed 
per episode, and $\gamma \in (0,1)$ is the discount factor. 
Equivalently under the circuit-supported action notation, given $\tri$ and a flippable circuit $\circuit$, the next state is uniquely determined as $\flip{\tri}{\circuit}$.
The realized subtriangulation $\tri^+(\circuit)$ specifies which direction of the bidirectional flip is applied.

The reinforcement learning objective is to learn a policy $\pi_f$ that chooses flips at each state to maximize the expected discounted cumulative reward over the horizon $H$:
\begin{equation}
  \mathcal{J}(\pi) := \mathbb{E} \left[\sum_{t=0}^{H-1} \gamma^t r(\tri_t,\circuit_t) \right], \quad
  \pi^\star = \argmax_\pi \mathcal{J}(\pi).
\end{equation}
The resulting policy is used to navigate the flip graph by selecting feasible bistellar flips from the current triangulation, with the aim of finding candidate triangulations that optimize the objective. 

\section{TriSearch}

%\vspace{-0.5em}
\subsection{Policy}

\begin{figure}[t]
    \centering
    \includegraphics[width=\textwidth]{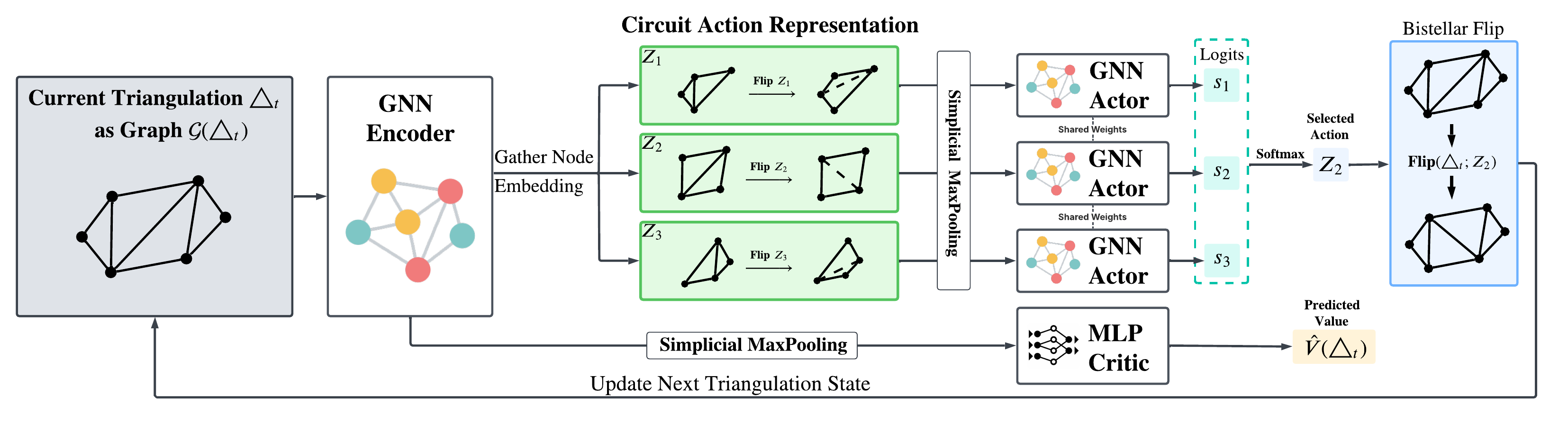}
    \caption{Architecture and training loop of \algoname. The shared encoder $\modelencoder$ maps the triangulation graph to node-level embeddings. 
    The actor scores each circuit that supports a flip $Z_i \in \flipspace{\tri}$ using a simplex-level GNN and selects a flip via softmax. The value head $\modelvalue$ max-pools the simplex embeddings from this GNN into a scalar estimate of the state value.}
    \label{fig:architecture}
\end{figure}
\vspace{-0.5em}

In this section, we present our policy architecture and optimization method for triangulation optimization.
At a state $\tri$, the policy must select a circuit from the variable-size action set $\flipspace{\tri}$ while conditioning on the full triangulation.
We therefore propose a two-stage architecture illustrated in Figure~\ref{fig:architecture}. 
A shared encoder $\modelencoder$, shared by both actor and value networks, processes the entire triangulation and produces vertex-level representations, 
providing each vertex with global geometric context. 
The actor $\modelactor$ then lifts these vertex embeddings to maximal simplices, propagates information over the $d$-simplicial graph,
and outputs a score for each flippable circuit.
A value head $\modelvalue$ max-pools the simplex embeddings from the actor and uses an MLP to estimate the state value for policy gradient training. 

For a triangulation $\tri$, let $G(\tri) = (\vpoly(\poly), \mathcal{E}(\tri))$ denote the graph whose nodes are the vertices of the polytope and whose edges are the $1$-skeleton of $\tri$. 
%We write the maximal simplices of $\tri$ as $\tri=\{\simplex_1,\ldots,\simplex_M\}$. 
For a flippable circuit $Z \in \flipspace{\tri}$, %we use the convention from Definition~\ref{def:flip}:
we let $\tri^+(Z)$ be the realized local subtriangulation contained in $\tri$, and let $\tri^-(Z)$ be its replacement.

\textbf{Encoder.} 
The triangulation problem is invariant under rigid motions of the %ambient 
polytope and does not depend on the %arbitrary 
ordering of the vertices. 
To respect these symmetries, the %shared 
encoder employs an EGNN~\cite{satorras2021n} to the full graph $G(\tri)$. 
For each vertex $p_i \in \vpoly(\poly)$, the coordinates serve as the geometric input and a learnable projection initializes the hidden features. After $L_{\mathrm{Enc}}$ layers, the encoder outputs vertex embeddings $\mathbf{h}^\star$ and updated coordinates $\mathbf{x}^\star$. 

\textbf{Actor.}  
The actor %is where the policy 
turns geometric context into a flip decision. The shared encoder works at the vertex level because relative coordinates relations such as affine dependence naturally live on $\vpoly(\poly)$. 
However, a feasible flip is not decided by one vertex or one edge. Instead, it replaces the realized local subtriangulation $\tri^+(Z)$ by $\tri^-(Z)$. So the useful action context is carried by $d$-simplices
and %by how they meet through facets. 
their adjacency relations. For this reason, we design the actor to leverage the geometric information extracted by the encoder to build features on maximal simplices and then pooling the features over $\tri^+(Z)$ for each feasible circuit and its realized subtriangulation.

Given the encoder embeddings $\mathbf{h}^\star$, 
the actor lifts the vertex embeddings to simplex level for each $d$-simplex $\simplex_a \in \tri$ by performing pooling over its vertices: 
\begin{equation}
  \mathbf{g}_{a}^{(0)}
  =
  \mathrm{MaxPool}\left(
    \{\mathbf{h}_i^\star : p_i \in \vpoly(\simplex_a)\}
  \right).
\end{equation}
The actor %then 
operates on %this 
the \emph{dual graph} $\hat{\mathcal{G}}(\tri)$ of the triangulation, where nodes are the $d$-simplices and two nodes are adjacent if the corresponding simplices share a $(d-1)$-face. 
Simplicial convolution layers~\cite{ebli2020simplicial} are %then 
used to refine the features along this graph. 
Simplex embeddings $\mathbf{g}_{a}^{\star}$ are produced for each $d$-simplex $\simplex_a \in \tri$ after $L_{\mathrm{\pi}}$ layers.
This propagation lets the score for a flip depend on the local arrangement around $\tri^+(Z)$, rather than only on the raw vertices in its circuit.  
Finally, for a feasible circuit $Z \in \flipspace{\tri}$, the actor pools the refined simplex features over the realized local subtriangulation $\tri^+(Z)$ and projects the result to one logit using a linear projection with learnable matrix $W$:
\begin{equation}
  \mathbf{z}_Z =
  \mathrm{MaxPool}\left(
    \{\mathbf{g}_{a}^{(L_{\mathrm{\pi}})} : \simplex_a \in \tri^+(Z)\}
  \right),
  \qquad
  s_{\actorparam}(\tri, Z) = W^{\top} \mathbf{z}_Z .
\end{equation}
The policy is the categorical distribution over the currently flippable circuits:
\begin{equation*}
  \modelactor(Z \mid \tri) = \frac{\exp ( s_{\actorparam}(\tri, Z) )}{\sum_{Z' \in \flipspace{\tri}} \exp ( s_{\actorparam}(\tri, Z'))} .
\end{equation*}
The resulting distribution compares only currently flippable circuits, which matches the implicit flip-graph setting. The encoder supplies reusable geometric features for the whole triangulation, and the actor specializes them to the local simplicial structure that determines each action. The ablation in Table~\ref{tab:ablation-controlled} supports this choice. Removing simplicial information substantially worsens the controlled $4$D results across all three objectives.

\textbf{Value Function.}  
We also use a light-weight value model $\modelvalue$, consisting of a pooling layer followed by an MLP $\phi_V$. 
It takes the refined simplex embeddings $\mathbf{g}_{a}^{(L_{\pi})}$ from the actor and aggregates them over all $d$-simplices of $\tri$ into a global state representation. 
The MLP then maps this representation to a scalar value estimate: 
\begin{equation}
  \modelvalue(\tri) = \phi_V\left(\mathrm{MaxPool}\left(\{\mathbf{g}_{a}^{(L_{\pi})} : \simplex_a \in \tri\}\right)\right).
\end{equation}
The value head therefore shares the simplicial layers with the actor. 
It estimates the state value from the same simplicial structure that the actor uses to score flips.

%\vspace{-0.6em}
\subsection{Policy Optimization}
%The implicit-graph setting also shapes training. 
Our goal is to have a policy to perform optimization across polytopes with different combinatorial and geometric configurations. To this end, we train the policy over a distribution of polytopes to learn the reusable navigation behavior that can generalize beyond the training data. Each polytope induces its own Markov Decision Process (MDP) and the parameters of encoder, actor and value function are shared across these MDPs. For each training polytope, we maintain a small set of its triangulations $\hat{V} \subseteq \trispace{\poly}$ as initial state. In each episode, we initialize our policy at one of the initial triangulation from $\hat{V}$ and generate rollouts. To prevent the model from overfitting and getting stuck at suboptimal trajectories, we add a count-based expansion bonus $b_t = \beta N(\tri_{t+1})^{-1/2}$, where $N(\tri)$ is the visitation count of triangulation state $\tri$ in the discovered training graph for $\poly$, initialized to one.
It encourages the agent to enlarge the observed portion of the implicit graph during training, as well as exploring new optimization strategies. 

\textbf{Weighted Initial State Sampling.} For training efficiency, we want to prevent \algoname{} from repetitively sampling from a well-explored region~\cite{ecoffet2019go}, which would gather biased samples from a concentrated region of the graph, causing the policy to get stuck in sub-optimality. Hence, we further apply $N(\tri_{0})^{-1/2}$ to a weighted initial state sampling: the initial states with higher bonus represents under-explored area, which will be sampled with higher weights.

\textbf{Training.} 
We train the encoder, actor, and value head jointly with PPO \cite{schulman2017proximal} using generalized advantage estimation (GAE)~\cite{schulman2015high}. 
In each iteration, the current policy collects $B$ rollout trajectories with transitions $(\tri_t, Z_t, \tilde r_t, \tri_{t+1})$, 
where $\tilde r_t = r_t + b_t$ includes the expansion bonus. We use the standard clipped policy loss and a squared value loss on $\modelvalue(\tri_t)$ against the empirical return. Advantages are estimated by GAE with smoothing parameter $\lambda$. The full algorithm is given in Algorithm~\ref{alg:train}, and the details of the PPO can be found in Appendix~\ref{app:training}.

\begin{algorithm}[tb]
  \caption{\algoname Training} \label{alg:train}
  \begin{algorithmic}
    \STATE {\bfseries Input:} Training polytopes $\mathcal{D}$, Seed Triangulations $\hat{V}_{\mathcal{D}} = \cup_{\poly \in \mathcal D}\hat{V}_{\poly}$, Objective $f$, Horizon $H$, Bonus coefficient $\beta$, Number of parallel environments $N$
    \STATE {\bfseries Initialize:} Model Parameters $(\encoderparam,\actorparam,\valueparam)$, and Visitation counts $N_{\poly}(\cdot) = 1$
    \FOR{training iteration $j = 1,2,\dots,J$}
      \STATE Sample $N$ initial triangulations: $\tri_0^{(i)} \sim \hat{V}_{\mathcal{D}}$ proportional to weight $N_{\poly}(\tri_0^{(i)})^{-1/2}$
      \STATE Initialize rollout buffer: $\mathcal{B} \gets \emptyset$
      \FOR{each rollout $i$ and step $t=0,\dots,H-1$}
        \STATE Enumerate flippable circuits $\flipspace{\tri_t^{(i)}}$ with TOPCOM
        \STATE Apply $\modelencoder$ to $G(\tri_t^{(i)})$ to obtain global representation: $h^{(i),\star}_t \gets \modelencoder(G(\tri_t^{(i)}))$
        \STATE Pool $h^{(i),\star}_t$ on $d$-simplex level, apply $\modelactor$ to select circuit to flip: $Z_t^{(i)} \in \flipspace{\tri_t^{(i)}}$
        \STATE Max-pool the refined simplex embeddings over $\tri_t^{(i)}$ to obtain the value estimate $\modelvalue(\tri_t^{(i)})$
        \STATE Flip the triangulation and obtain the next state: $\tri_{t+1}^{(i)} \gets \flip{\tri_t^{(i)}}{Z_t^{(i)}}$
        \STATE Compute the reward $r$ and store the transition $(\tri_t^{(i)}, Z_t^{(i)}, r_t^{(i)}, \tri_{t+1}^{(i)})$ to rollout buffer $\mathcal{B}$
      \ENDFOR
      \STATE Update $(\encoderparam,\actorparam,\valueparam)$ using PPO with rollout buffer $\mathcal{B}$
    \ENDFOR
  \end{algorithmic}
\end{algorithm}

\section{Experiments}
\label{sec:experiments}

We design our experiments around three questions about \algoname{}.
First, we ask whether the learned policies by \algoname{} can solve diverse triangulation objectives in both $3$D and $4$D.
Second, we ask whether the policy generalizes from small training polytopes to larger unseen ones, where the flip graph is exponentially larger.
Third, we ask whether the same framework supports a real downstream pipeline beyond benchmark optimization.
We address the first two questions in Section~\ref{sec:exp3d} through triangulation optimization against classical and learned local-search baselines.
We address the third question in Section~\ref{sec:exp4d} by reusing the training pipeline for FRST discovery on $4$D reflexive polytopes from the Kreuzer-Skarke list~\cite{kreuzer2000complete}, and then applying the trained model to Calabi-Yau threefold sampling against the \textsc{CYTools} samplers~\cite{demirtas2022cytools}.
Our experiments were run with a \textsc{RTX 2080} Ti GPU and Intel Xeon CPU.
Throughout the experiments, all MLPs and GNNs have 64 hidden units in each layer, with 3 layers in encoder and value function and two layers for GNN actor, use SiLU activation function~\cite{ramachandran2017searching}, and are trained using the Adam optimizer with a learning rate of 0.0001~\cite{kingma2014adam}.
The other detailed hyperparameters can be found in Appendix~\ref{app:hyperparameters}.

\begin{figure*}[t]
\centering
\makebox[\textwidth][c]{%
  \includegraphics[width=1.0\textwidth,height=0.92\textheight,keepaspectratio]{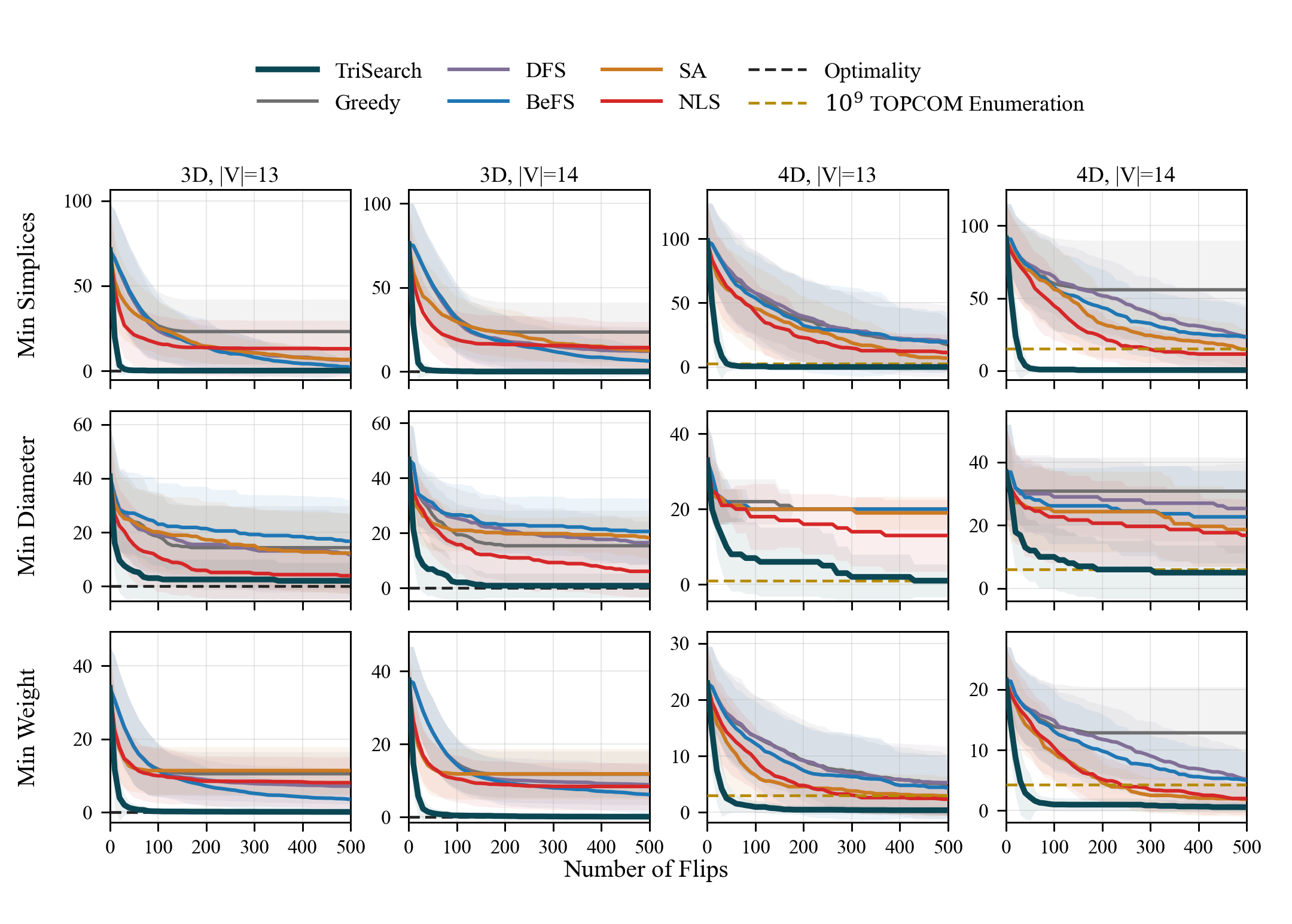}%
}
\vspace{-2em}
\caption{
  Triangulation optimization under a $500$-flip budget. 
  It shows the mean relative gap and standard deviation over evaluation polytopes. 
  The black dashed line is exact optimality in $3$D, and the gold dashed line is $10^9$ \textsc{TOPCOM} enumeration used as an external $4$D baseline. 
}
\label{fig:dense-optimization-500}
\end{figure*}

%\vspace{-0.6em}
\subsection{Triangulation Optimization}
\label{sec:exp3d}
In this setting, we aim to evaluate the capacity of \algoname{} in triangulation optimization across various objectives, and verify its out-of-distribution generalization across polytope size.
For each dimension $d \in \{3, 4\}$, we sample $130$ training polytopes with $8$ to $11$ vertices and generate up to $2000$ seed triangulations per polytope with \textsc{TOPCOM}. The detailed data generation procedure is provided in Appendix~\ref{app:synthetic-data}.
Evaluation uses held-out polytopes with $13$ or $14$ vertices, where the flip graph is exponentially larger than what the policy saw at training time. 
For each dimension, we train one policy per objective for three objectives: \textbf{Min Simplices} (number of $d$-simplices), \textbf{Min Diameter} (diameter of the dual graph), and \textbf{Min Weight} (total edge length).
We compare against four classical heuristics-based baselines (\textsc{Greedy}, \textsc{DFS}, \textsc{BeFS}, \textsc{SA}), which select the flip based on objective improvements, and the learned meta-heuristics controller \textsc{NLS}~\cite{falkner2022neurols}, which learns a state-conditioned acceptance rule on random flip proposals.
For fair comparison, \textsc{NLS} shares the same EGNN architecture and the training hyperparameters as \algoname{}. The more detailed descriptions of the baseline algorithms are deferred to Appendix~\ref{app:baseline-details}. 
Each method receives a budget of $500$ flips from the same starting triangulations, and we report the average relative gap as:
\begin{equation*}
  \mathrm{RelativeGap}(A)
  =
  \frac{1}{|S|}
  \sum_{\poly \in S}
  \frac{f(\hat{\tri}_A(\poly)) - f(\tri_{\mathrm{ref}}(\poly))}
       {f(\tri_{\mathrm{ref}}(\poly))},
\end{equation*}
where $\hat{\tri}_A(\poly)$ is the best triangulation found by $A$ within the budget or from \textsc{TOPCOM} enumeration. 
In $3$D, the reference $\tri_{\mathrm{ref}}(\poly)$ is the exact optimum found by \emph{exhaustive} \textsc{TOPCOM} enumeration. In $4$D, an exhaustive enumeration is not feasible at this size, we perform a $10^9$-steps enumeration with \textsc{TOPCOM} as an additional baseline, and use the best triangulation found across all algorithms tested as the reference $\tri_{\mathrm{ref}}(\poly)$.

\begin{wrapfigure}[18]{r}{0.44\textwidth}
%\vspace{-1.0\baselineskip}
\centering
\subcaptionbox{\label{fig:landscape-surface}}[0.52\linewidth]{%
    \includegraphics[width=\linewidth]{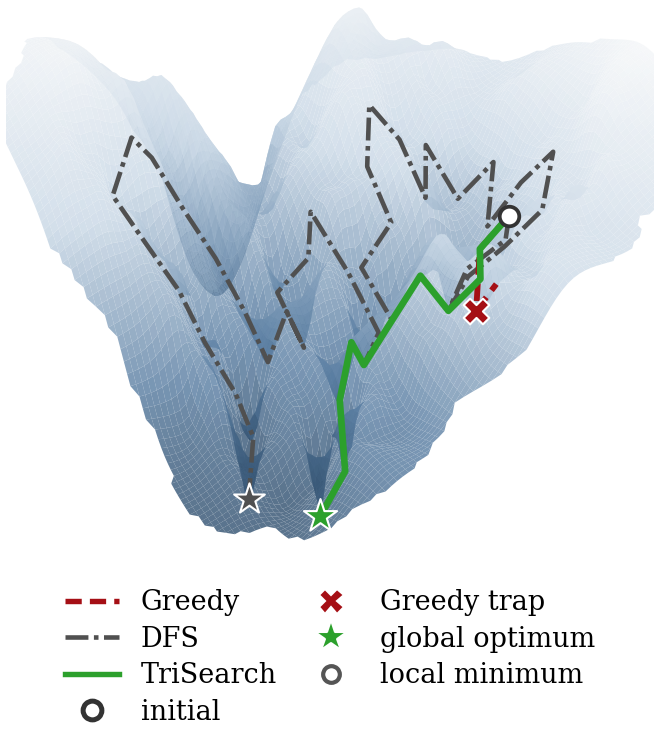}%
}%
\hfill
\subcaptionbox{\label{fig:landscape-curve}}[0.46\linewidth]{%
    \includegraphics[width=\linewidth]{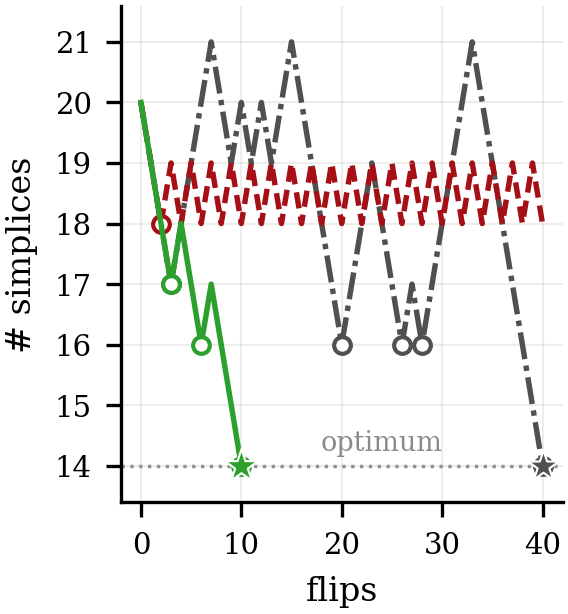}%
}
\caption{Min Simplices trajectories on an unseen $3$D polytope with $13$ vertices. (a) Flip subgraph with heights given by the number of simplices. (b) Number of simplices along each trajectory.}
\label{fig:landscape}
\end{wrapfigure}

Figure~\ref{fig:dense-optimization-500} reports the optimization curves on the held-out polytopes. Notably, \algoname{} surpasses the $10^9$-steps \textsc{TOPCOM} reference on \textbf{Min Simplices} and \textbf{Min Weight} at both $|\vpoly|=13$ and $|\vpoly|=14$ within $500$ flips in $4$D.
\algoname{} largely closes the gap within around $100$ flips on every objective and reaches an average gap of $0.56\%$ in $3$D and $1.21\%$ in $4$D, against $8.97\%$ and $9.48\%$ for the next-best baseline \textsc{NLS}. The detailed aggregated result tables can be found in Appendix~\ref{app:additional-experiments}.

\textbf{Case study.} Figure~\ref{fig:landscape} compares \textsc{Greedy}, \textsc{DFS} and \algoname{} on one unseen $3$D polytope with $13$ vertices.
Its full flip graph has about $9.5$M triangulations and $56.9$M flips.
Figure~\ref{fig:landscape-surface} places the visited triangulations and their $2$-step neighbors in the plane by multidimensional scaling on flip distance.
The height is the number of simplices, and the surface is interpolated only as a visual guide.
\textsc{Greedy} is trapped at a local minimum with $18$ simplices.
\textsc{DFS} passes several local minima and reaches an optimum after $40$ flips.
\algoname{} accepts a few temporary increases and reaches another optimum within $10$ flips.

% \noindent
% \begin{minipage}[t]{0.42\textwidth}
% \vspace{0pt}
\textbf{Ablation study.} We ablate the actor on the shared $4$D evaluation splits with $|\vpoly|=13$ and $|\vpoly|=14$.
The full \algoname{} model uses an SNN actor to score each flip from its local subtriangulation.
% \end{minipage}\hfill

\begin{wraptable}[7]{r}{0.6\textwidth}
    \centering
% \begin{minipage}[t]{0.56\textwidth}
%\vspace{-0.8\baselineskip}
\centering
\begingroup
\captionsetup{font=small,skip=2pt,hypcap=false}
\captionof{table}{Controlled $4$D ablation at $500$ flips. Entries are average relative gap with standard error (\%, $\downarrow$).}
\label{tab:ablation-controlled}
\scriptsize
\setlength{\tabcolsep}{2pt}
\renewcommand{\arraystretch}{0.95}
\resizebox{\linewidth}{!}{%
\begin{tabular}{@{}lccc@{}}
\toprule
Variant & Min Simp. & Min Diameter & Min Weight \\
\midrule
\algoname{} (SNN) & \textbf{0.185{\tiny$\pm$0.183}} & \textbf{3.000{\tiny$\pm$1.129}} & \textbf{0.580{\tiny$\pm$0.144}} \\
w/o simplicial information (EGNN) & 9.196{\tiny$\pm$2.207} & 13.500{\tiny$\pm$1.481} & 3.009{\tiny$\pm$0.664} \\
w/o subtriangulation (Pooling + MLP) & 15.303{\tiny$\pm$2.720} & 15.333{\tiny$\pm$1.794} & 6.884{\tiny$\pm$0.654} \\
\bottomrule
\end{tabular}%
}
\endgroup
% \vspace{-0.7\baselineskip}
% \end{minipage}
% \par\smallskip
% \noindent
\end{wraptable} 

The EGNN variant removes simplicial information, and the Pooling + MLP variant removes subtriangulation structure altogether.
Table~\ref{tab:ablation-controlled} reports the average relative gap over $40$ evaluation polytopes for each objective.
The reference includes the $4$D \textsc{TOPCOM} enumeration value and every $500$-flip method in the main and ablation result files.
Both removals increase the relative gap across all three objectives, indicating the necessity of $d$-simplicial information when making a flip decision.
%\vspace{-0.6em}

\subsection{Calabi-Yau Threefold Sampling}
\label{sec:exp4d}

\begin{figure*}[t]
\centering
\subcaptionbox{\label{fig:cy_success_rate}}[0.244\textwidth]{%
    \includegraphics[width=\linewidth]{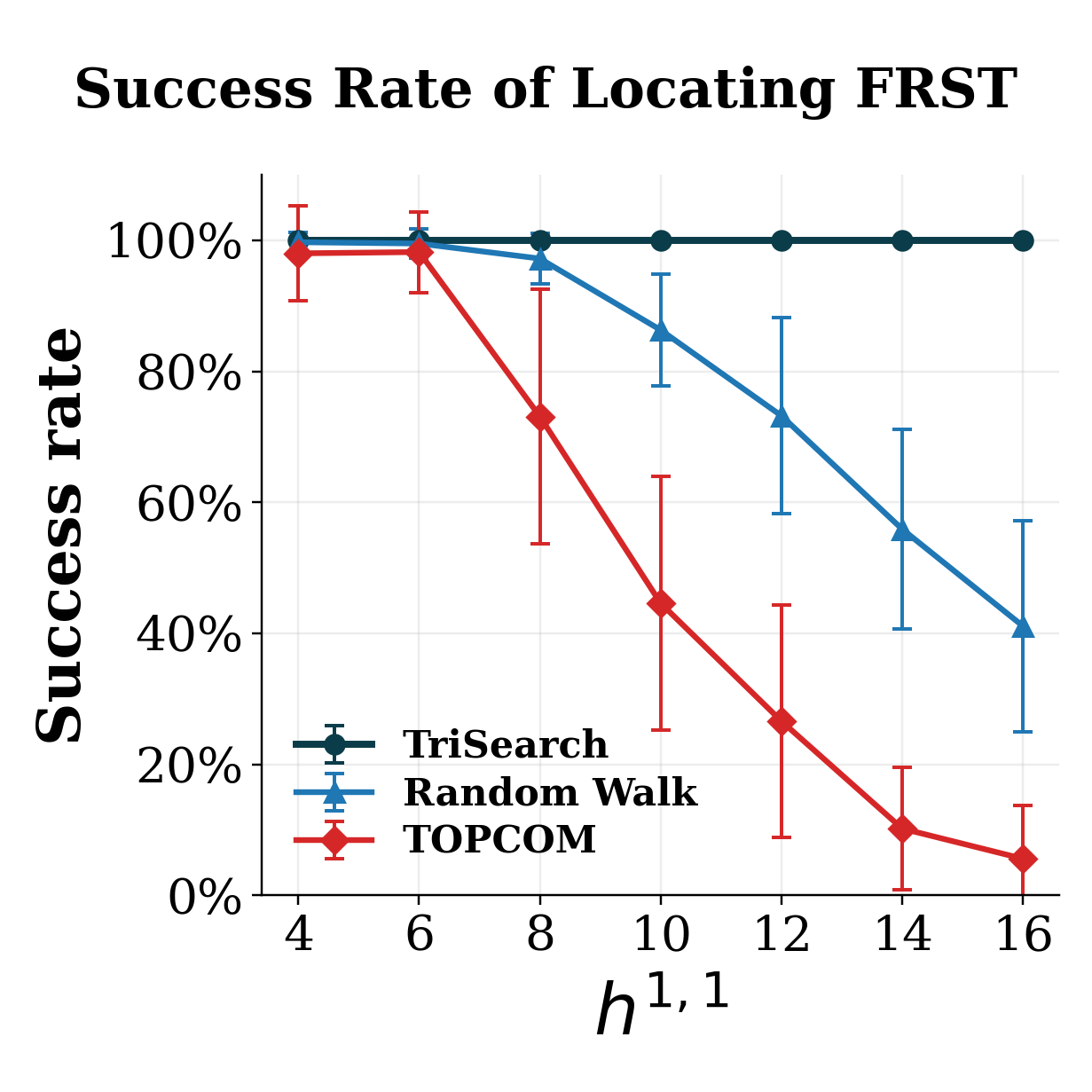}%
}
% \hfill
\subcaptionbox{\label{fig:frst_over_time}}[0.244\textwidth]{%
    \includegraphics[width=\linewidth]{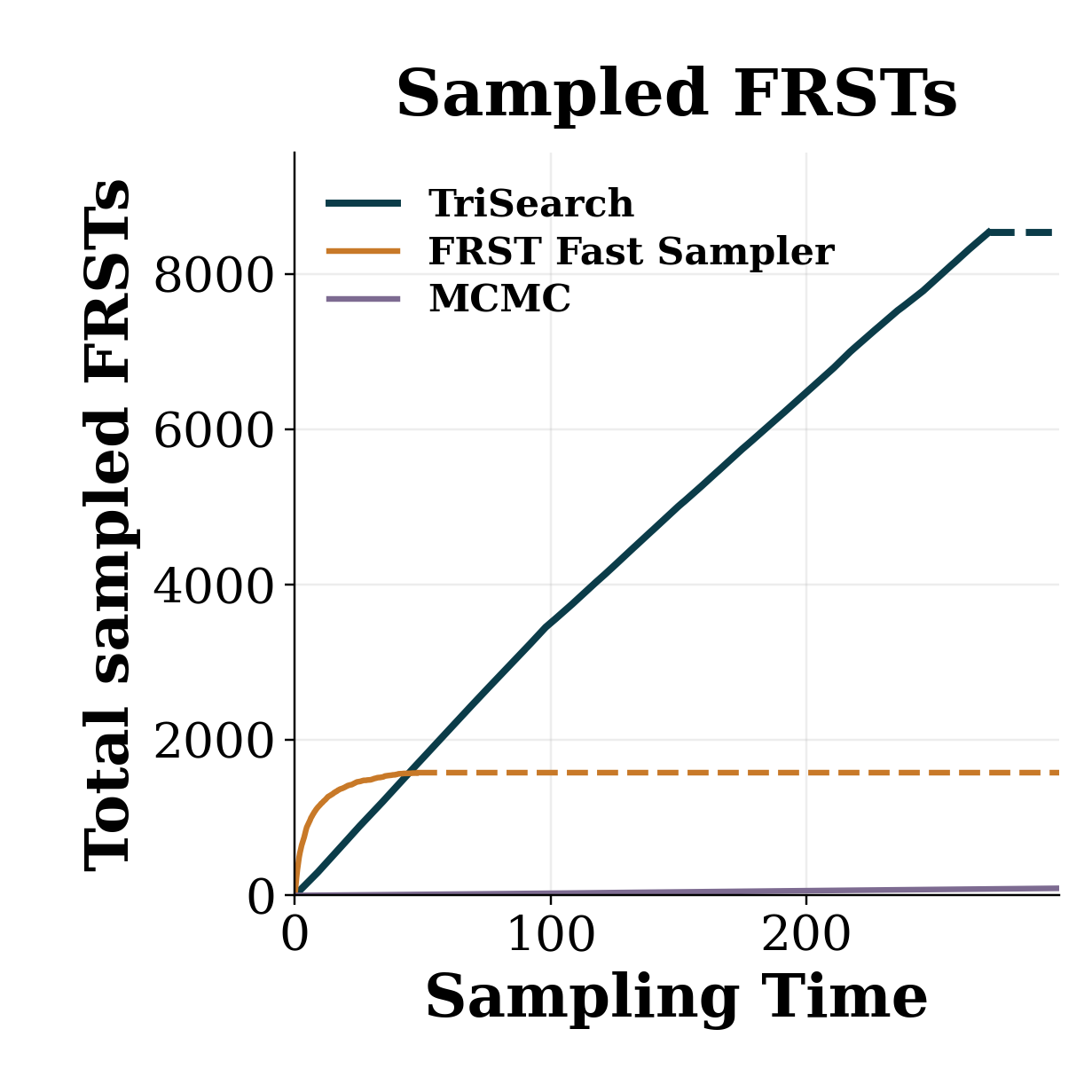}%
}
% % \hfill
\subcaptionbox{\label{fig:cy_class_over_time}}[0.244\textwidth]{%
    \includegraphics[width=\linewidth]{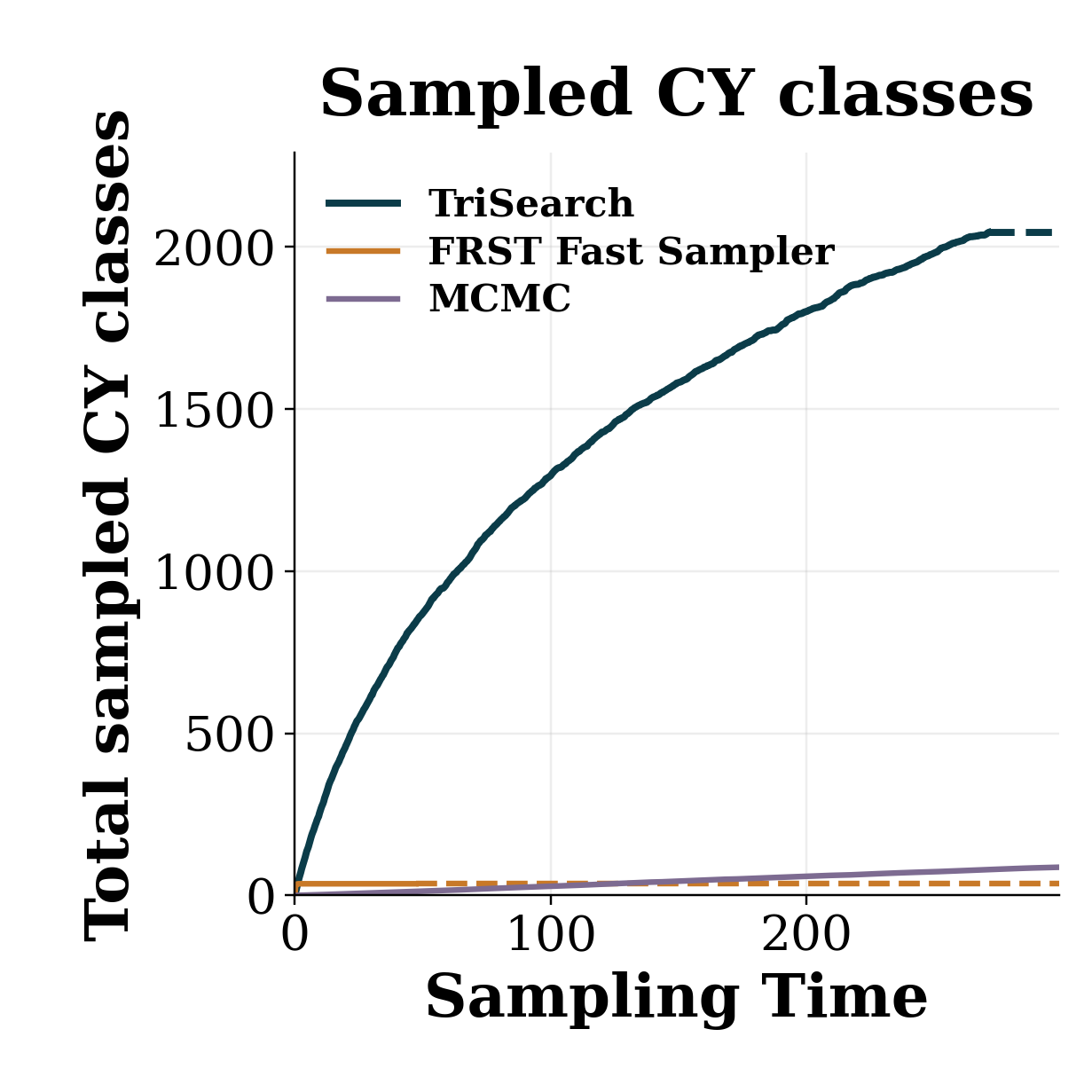}%
}
% % \hfill
\subcaptionbox{\label{fig:cy_class_fraction}}[0.244\textwidth]{%
    \includegraphics[width=\linewidth]{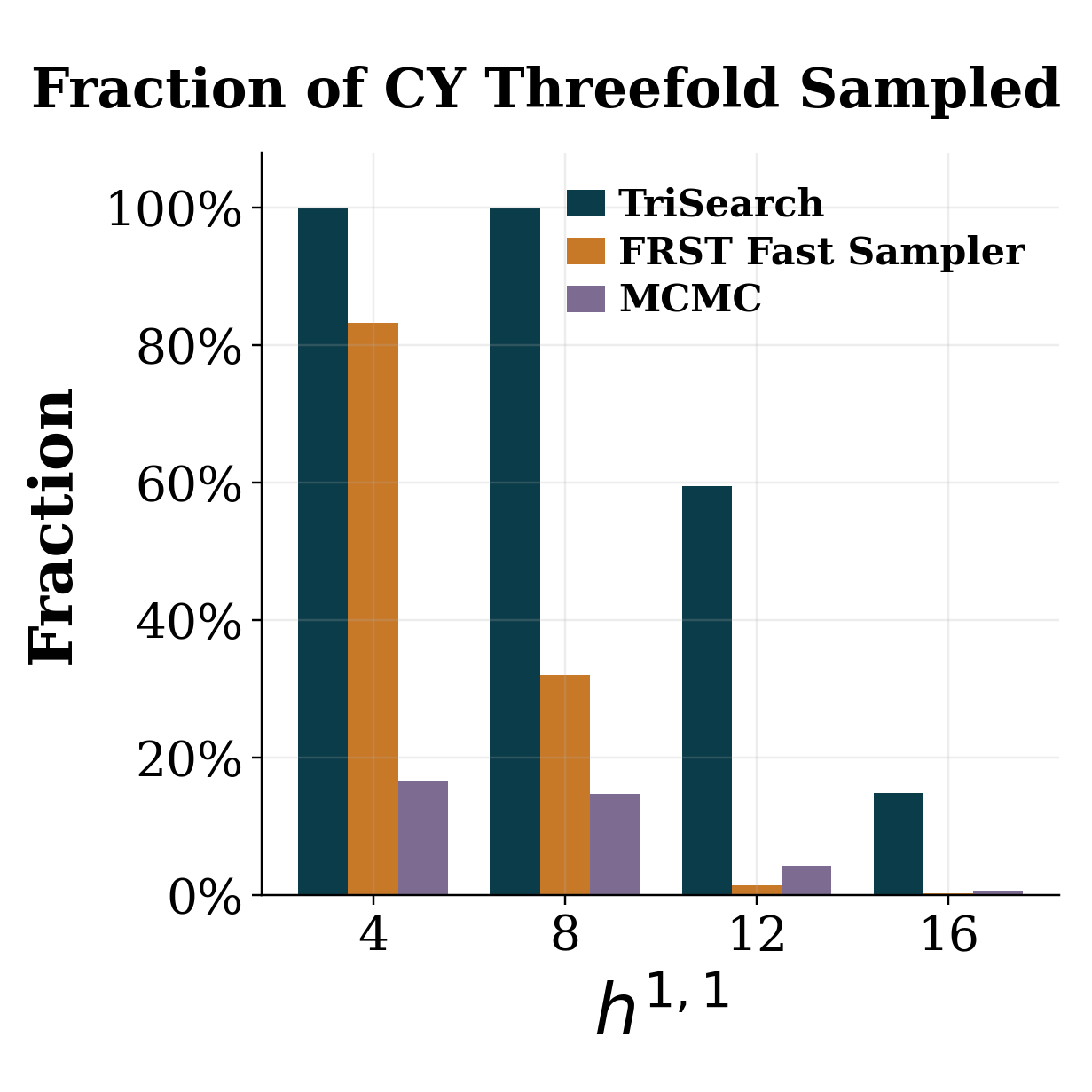}%
}
\caption{CY threefold sampling experiments.
(a) Success rate of locating a nearby FRST within a $50$-flip budget on unseen polytopes with $h^{1,1}\in\{4,6,\ldots,16\}$.
Error bars show standard deviations across $20$ polytopes.
(b, c) Cumulative numbers of unique FRSTs (b) and CY threefolds (c) discovered over sampling time at $h^{1,1}=16$. The dashed line represents that the stopping criterion was met. \algoname{} stopped at the maximum iteration, while \textsc{FRST Fast Sampler} was stopped by hitting max retries.
(d) Fraction of CY threefold classes discovered for $h^{1,1}\in\{4,8,12,16\}$. Labels give the discovered class count and the percentage.} \vspace{-1.0em}
\label{fig:cy_sampling}
\end{figure*}

We next evaluate the %same 
framework on CY threefold sampling. 
At the geometry level, our task is to find Fine, Regular, Star triangulations (FRSTs) of 4D reflexive lattice polytopes, then use the Batyrev construction to map each FRST %of a $4$D reflexive polytope 
to a CY threefold, where different FRSTs may yield topologically equivalent threefolds.\footnote{Recall that a triangulation is fine if it uses every lattice point, regular if it can be induced by a lifting, and star if all maximal simplices contain the origin.} 
The practical goal is therefore to discover as many distinct threefold classes as possible. 
We formulate this as a nearby-FRST discovery task from a random starting triangulation with sparse binary reward. 
We train the policy to start from a regular triangulation induced by random heights and may take at most $50$ flips to locate nearby target triangulations. 
It succeeds and receives a positive reward $+1$ when it reaches a fine regular triangulation, 
after which standard post-processing closes it into a star triangulation~\cite{demirtas2020bounding}. 
We train \algoname{} on small reflexive polytopes with Hodge number $h^{1,1}=3$, and evaluate on $1000$ random triangulations of unseen polytopes with $h^{1,1}$ from $4$ to $16$. 
Figure~\ref{fig:cy_success_rate} shows that \algoname{} maintains a perfect success rate across the range, while \textsc{TOPCOM} breadth-first search and \textsc{Random Walk} degrade quickly as $h^{1,1}$ grows. 

We then turn the trained FRST locator into a CY sampler. 
In each iteration, we draw a random height vector in $\mathbb{R}^{|\vpoly(\poly)|}$ from a diagonal Gaussian distribution, construct the induced regular triangulation through lifting, and run \algoname{} to search for a nearby FRST. 
Sampling efficiency is evaluated on $10$ polytopes each for $h^{1,1}\in\{4,8,12,16\}$. 
We compare against the widely used software \textsc{CYTools}~\cite{demirtas2022cytools, demirtas2020bounding}, including \textsc{FRST Fast Sampler} and \textsc{MCMC}, under the same budgeted regime. 
Each sampler is assigned at most $300$ seconds or $1024$ iterations per polytope, and sampling stops after $50$ consecutive retries without discovering a new FRST. 
Figures~\ref{fig:frst_over_time} and~\ref{fig:cy_class_over_time} show the sampling count at $h^{1,1}=16$.  \algoname{} exhibits nearly linear FRST growth and discovers significantly more distinct CY classes than baseline samplers. 
\textsc{MCMC} also continues to discover new FRSTs and CY classes, but suffers from low sampling efficiency, 
while \textsc{FRST Fast Sampler} initially achieves very high FRST sampling speed before rapidly decaying and hitting max retries. 
Figure~\ref{fig:cy_class_fraction} aggregates the final class fractions across $h^{1,1}\in\{4,8,12,16\}$. 
\algoname{} recovers all known classes at $h^{1,1}=4$ and $h^{1,1}=8$, and outperforms both \textsc{CYTools} samplers at $h^{1,1}=12$ and $h^{1,1}=16$.

\section{Conclusion}
We introduced \algoname, a reinforcement-learning framework for optimizing and exploring triangulations of polytopes in any fixed dimension, where the key design element is a circuit-supported subtriangulation action representation: geometric routines enumerate valid bistellar flips, and the policy learns how to rank them using both global triangulation context and the realized local subtriangulation. TriSearch outperforms all the baselines under the same budget across $3$D and $4$D triangulation optimization, and outperforms the widely used samplers in \textsc{CYTools} across $4$D CY threefold sampling task.
These results support learned flip-graph navigation as an effective approach to combinatorial geometry problems whose search spaces are too large to enumerate, yet whose local moves are mathematically well formulated.
Similar high-level structure also appears in phylogenetic tree search \cite{liptak2021constructing}, rhombus tiling \cite{chavanon2006rhombus}, and unknotting problems \cite{shehper2024makes}, where the feasible space is too large to enumerate but local moves expose rich mathematical structure. 
% Despite the promising empirical results shown by \algoname{}, we identify three limitations in our current work and we aim to address in the future: (1) Our method relies on a geometric routine to enumerate flippable circuits that cannot be parallelized, (2) Our method lack a mechanism to encounter the situation where the flip graph is disconnected. (3) In the CY sampling application, we sample the random starting triangulations uniformly, which can be potentially improved with insights on Reflexive polytopes and CY threefolds.

Despite the strong empirical performance of \algoname{}, our current work has several limitations: First, the method relies on a geometric routine to enumerate flippable circuits, which is currently difficult to parallelize; future work could investigate replacing this component with a neural network. 
Second, the method does not currently address situations in which the flip graph is disconnected. 
Third, in the CY sampling application, the random starting triangulations are sampled uniformly. Future work could instead learn the distribution of FRSTs in height space and sample more strategically.

\section*{Acknowledgments} 

We thank Dr.~Andres Rios-Tascon from Princeton University and Dr.~Nate MacFadden from Cornell University for valuable discussion regarding Calabi-Yau manifold sampling. 

This project has been supported by 
the NSF program Artificial Intelligence, Formal Methods, and Mathematical Reasoning (AIMing) under grant 2522495. 
GM was partially supported by 
DARPA AIQ grant HR00112520014, 
NSF grants 
%DMS-2522495,
2446222, 
DMS-2145630, 
CCF-2212520, 
DFG SPP 2298 grant 464109215, 
and BMFTR in DAAD project 57616814 (SECAI). 

\newpage 
\bibliographystyle{unsrt}
\bibliography{references}

\clearpage
\appendix

\section{Radon Partitions and Circuit Flips}
\label{app:radon-circuits}

This appendix records the standard construction behind the two local triangulations used in Definition~\ref{def:flip}.
Let $\poly$ have vertex set $\vpoly$ and let $\circuit \subset \vpoly$ be a circuit.
When forming links and joins, we identify a simplex with its vertex set.
Since $\circuit$ is minimally affinely dependent, there is an affine dependence relation, unique up to nonzero scaling,
\[
\sum_{p_i \in \circuit} \lambda_i p_i = 0,
\qquad
\sum_{p_i \in \circuit} \lambda_i = 0,
\]
where all coefficients are nonzero and both signs occur.
The signs define the Radon partition
\[
\circuit^+ = \{p_i \in \circuit : \lambda_i > 0\},
\qquad
\circuit^- = \{p_i \in \circuit : \lambda_i < 0\},
\]
and the two parts have intersecting convex hulls:
\[
\conv{\circuit^+} \cap \conv{\circuit^-} \neq \emptyset.
\]
Let $r=\dim \conv{\circuit}$.
The circuit therefore admits two sets of maximal $r$-simplices in $\conv{\circuit}$:
\begin{align*}
\tri_{\mathrm{core}}^+(\circuit) &:= \{\circuit \setminus \{p\}: p \in \circuit^+\}, \\
\tri_{\mathrm{core}}^-(\circuit) &:= \{\circuit \setminus \{p\}: p \in \circuit^-\}.
\end{align*}
We write these cores by their maximal faces, following the convention in the main text.
Whenever lower-dimensional faces are needed, we pass to the corresponding simplicial closure.
Changing the sign of the affine dependence swaps the $+$ and $-$ labels.

When $r=d$, the core triangulations are full-dimensional.
This is the general-position case.
We set $L_{\tri}(\circuit)=\{\emptyset\}$, and the bistellar flip directly replaces $\tri_{\mathrm{core}}^+(\circuit)$ by $\tri_{\mathrm{core}}^-(\circuit)$.

When $r<d$, the circuit core must be embedded into the surrounding triangulation through a link.
Since $\tri$ is written by its maximal simplices, let $\overline{\tri}$ denote its simplicial closure:
\[
\overline{\tri}
=
\{F\subseteq \vpoly(\poly):F\subseteq \vpoly(\simplex) \text{ for some } \simplex\in\tri\}.
\]
For a face $F\in\overline{\tri}$, its link in $\tri$ is
\[
\Lk_{\tri}(F)=\{G\in\overline{\tri}:G\cap F=\emptyset,\ F\cup G\in\overline{\tri}\}.
\]
The link lives in $\overline{\tri}$ and uses vertices outside $F$.

For a flippable circuit, orient the signs so that every maximal face in $\tri_{\mathrm{core}}^+(\circuit)$ lies in $\overline{\tri}$.
The common link is
\[
L_{\tri}(\circuit)
=
\bigcap_{p\in \circuit^+}\Lk_{\tri}(\circuit\setminus\{p\}).
\]
Equivalently, every maximal face $\circuit\setminus\{p\}$ with $p\in \circuit^+$ has the same link $L_{\tri}(\circuit)$.
This link lives on vertices outside $\circuit$ and has maximal faces of dimension $d-r-1$.
Joining it with the $r$-dimensional circuit core gives $d$-simplices.
For two collections of faces $K$ and $L$ on disjoint vertices, their join is
\[
K*L=\{F\cup G:F\in K,\ G\in L\}.
\]
When we write $\tri_{\mathrm{core}}^\pm(\circuit)*L_{\tri}(\circuit)$ as part of $\tri$, we mean the set of maximal $d$-simplices obtained by taking $F\in\tri_{\mathrm{core}}^\pm(\circuit)$ and $G$ a maximal face of $L_{\tri}(\circuit)$.
The notation in the main text suppresses the link by writing the full-dimensional local subtriangulations as
\[
\tri^+(\circuit)
=
\tri_{\mathrm{core}}^+(\circuit)*L_{\tri}(\circuit),
\qquad
\tri^-(\circuit)
=
\tri_{\mathrm{core}}^-(\circuit)*L_{\tri}(\circuit).
\]
The full bistellar flip supported by $\circuit$ replaces the first joined complex by the second.
At the level of maximal $d$-simplices, it produces
\[
\tri'=\left(\tri\setminus\tri^+(\circuit)\right)\cup\tri^-(\circuit).
\]
This recovers the direct replacement in Definition~\ref{def:flip}.

\begin{proposition}[Uniqueness of the flip supported by a fixed circuit]
Let $\tri$ be a geometric triangulation of $\poly$ with point configuration $\vpoly(\poly)$.
Let $\circuit\subset \vpoly(\poly)$ be a circuit.
If $\tri_2$ and $\tri_3$ are both obtained from $\tri$ by bistellar flips supported on $\circuit$, then
\[
\tri_2=\tri_3.
\]
Equivalently, for fixed $\tri$ and fixed $\circuit$, there is at most one successor triangulation $\flip{\tri}{\circuit}$.
\end{proposition}

\begin{proof} 
%\noindent\textbf{Proof.}
Suppose $\tri_2$ and $\tri_3$ are two successors obtained by flips supported on $\circuit$.
Without loss of generality, orient the Radon partition so that the first flip removes the $+$ side of $\circuit$.
Thus there is a link $L_2$ such that, at the level of maximal $d$-simplices,
\[
\tri_2
=
\Bigl(\tri\setminus(\tri_{\mathrm{core}}^+(\circuit)*L_2)\Bigr)
\cup
\bigl(\tri_{\mathrm{core}}^-(\circuit)*L_2\bigr).
\]
We first show that the second flip must also remove the $+$ side.
Assume the opposite.
Then the second flip removes the $-$ side for some link $L_3$.
Because $\emptyset$ belongs to every link, this would imply that both $\tri_{\mathrm{core}}^+(\circuit)$ and $\tri_{\mathrm{core}}^-(\circuit)$ are contained in $\overline{\tri}$.
We show that this is impossible in a geometric triangulation.

Choose $p\in \circuit^+$ and $q\in \circuit^-$.
Then $\circuit\setminus\{p\}$ is a maximal face of $\tri_{\mathrm{core}}^+(\circuit)$, while $\circuit\setminus\{q\}$ is a maximal face of $\tri_{\mathrm{core}}^-(\circuit)$.
The affine dependence on the circuit gives a Radon point
\[
x\in \conv{\circuit^+}\cap\conv{\circuit^-}.
\]
Thus
\[
x\in \conv{\circuit\setminus\{p\}}
\cap
\conv{\circuit\setminus\{q\}}.
\]
However, $x\notin \conv{\circuit\setminus\{p,q\}}$.
Otherwise $x$ would have two affine representations in the affinely independent set $\circuit\setminus\{q\}$.
One uses $\circuit^+$ with a positive coefficient on $p$, and the other uses $\circuit\setminus\{p,q\}$ with coefficient $0$ on $p$.
This contradicts that every proper subset of a circuit is affinely independent.
Therefore the two simplices $\circuit\setminus\{p\}$ and $\circuit\setminus\{q\}$ do not meet in their common face.
They cannot both be faces of $\overline{\tri}$.
Thus the second flip cannot remove the $-$ side.

Therefore both flips remove the realized core $\tri_{\mathrm{core}}^+(\circuit)$. 
By the linked construction above, the link is then uniquely determined by $\tri$ and $\circuit$: 
\[
L_2
=
L_3
=
L_{\tri}(\circuit)
=
\bigcap_{p\in \circuit^+}\Lk_{\tri}(\circuit\setminus\{p\}).
\]
The removed joined complexes are identical, and the inserted joined complexes are identical.
Hence $\tri_2=\tri_3$. 
%\(\square\)
\end{proof}

\section{Architecture and Training Details}
\label{app:architecture}

\subsection{EGNN Encoder Details}

The encoder applies EGNN message passing to the graph $G(\tri)$ induced by the $1$-skeleton of the current triangulation.
For each vertex $p_i \in \vpoly(\poly)$, the initial coordinate and hidden feature are
\[
  \mathbf{x}_i^{(1)} := p_i,
  \qquad
  \mathbf{h}_i^{(1)} := W p_i .
\]
At layer $\ell$, the encoder computes an edge message for each neighbor $j \in \mathcal{N}(i)$:
\[
\mathbf{m}_{ij}
=
\phi_e\!\bigl(
\mathbf{h}_i^{(\ell)},
\mathbf{h}_j^{(\ell)},
\|\mathbf{x}_i^{(\ell)} - \mathbf{x}_j^{(\ell)}\|^2
\bigr).
\]
It then updates the coordinate and hidden feature by
\[
\mathbf{x}_i^{(\ell+1)}
=
\mathbf{x}_i^{(\ell)}
+
\frac{1}{|\mathcal{N}(i)|}
\sum_{j \in \mathcal{N}(i)}
\bigl(\mathbf{x}_i^{(\ell)} - \mathbf{x}_j^{(\ell)}\bigr)
\phi_x(\mathbf{m}_{ij}),
\]
\[
\mathbf{h}_i^{(\ell+1)}
=
\mathbf{h}_i^{(\ell)}
+
\phi_h\!\Bigl(
\mathbf{h}_i^{(\ell)},
\sum_{j \in \mathcal{N}(i)} \mathbf{m}_{ij}
\Bigr).
\]
Here $\phi_e$, $\phi_x$, and $\phi_h$ are MLPs.
After $L_{\mathrm{Enc}}$ layers, the encoder outputs vertex embeddings $\mathbf{h}^\star$ and updated coordinates $\mathbf{x}^\star$.

\subsection{Circuit-Supported Action Scoring}

The actor scores circuit-supported flips through their realized full-dimensional local subtriangulations.
It first lifts vertex features to maximal simplices.
For each $\simplex_a \in \tri$, it pools the encoder embeddings of the vertices in that simplex:
\[
  \mathbf{g}_{a}^{(0)}
  =
  \mathrm{MaxPool}\left(
    \{\mathbf{h}_i^\star : p_i \in \vpoly(\simplex_a)\}
  \right).
\]
Collecting these features gives $G^{(0)} \in \mathbb{R}^{M \times m}$, where each row corresponds to one maximal simplex in $\tri=\{\simplex_1,\ldots,\simplex_M\}$.
The actor then propagates information over maximal simplices rather than only over vertices.
This step gives each action score access to how the simplices in $\tri^+(Z)$ meet nearby simplices through shared facets.

Let $B_d$ be the oriented boundary matrix from $d$-simplices to $(d-1)$-faces.
The top-degree down Laplacian is
\[
  L_d = B_d^\top B_d .
\]
It connects two maximal simplices when they share a facet.
In the implementation, we normalize this sparse Laplacian by a global row-sum scale for numerical stability and denote the result by $\widetilde{L}_d$.
We use the simplicial propagation rule proposed by Ebli et al.~\cite{ebli2020simplicial}.
The Chebyshev recursion is defined by
\[
  T_0(G)=G,
  \qquad
  T_1(G)=\widetilde{L}_dG,
  \qquad
  T_k(G)=2\widetilde{L}_dT_{k-1}(G)-T_{k-2}(G)
  \quad k \geq 2.
\]
Each action-scoring propagation layer updates simplex features by
\[
  G^{(\ell+1)}
  =
  \sigma\left(
    \sum_{k=0}^{K_{\mathrm{act}}-1}
    T_k(G^{(\ell)})\Theta_k^{(\ell)}
    + \mathbf{b}^{(\ell)}
  \right),
\]
where the final layer omits $\sigma$.
For a circuit-supported flip $Z \in \flipspace{\tri}$, the actor pools features over the realized full-dimensional local subtriangulation $\tri^+(Z)$:
\[
  \mathbf{z}_Z =
  \mathrm{MaxPool}\left(
    \{\mathbf{g}_{a}^{(L_{\pi})} : \simplex_a \in \tri^+(Z)\}
  \right),
  \qquad
  s_{\actorparam}(\tri, Z) = W^\top \mathbf{z}_Z.
\]
The policy normalizes these logits across the current feasible action set $\flipspace{\tri}$.

\subsection{PPO Training Objective}
\label{app:training}

For each rollout transition, the training reward is $\tilde r_t = r_t + b_t$, where $b_t$ is the expansion bonus from Algorithm~\ref{alg:train}.
The PPO probability ratio is
\[
  \rho_t(\actorparam) :=
  \frac{\modelactor(Z_t \mid \tri_t)}
       {\pi_{\theta_{\mathrm{old}}}(Z_t \mid \tri_t)} .
\]
Generalized advantage estimation uses the temporal-difference residual
\[
  \delta_t =
  \tilde r_t + \gamma \modelvalue(\tri_{t+1}) - \modelvalue(\tri_t),
\]
with the bootstrap term omitted at the episode horizon.
The advantage estimate is
\[
  \hat{A}_t
  =
  \sum_{h=0}^{H-1-t} (\gamma \lambda)^h \delta_{t+h}.
\]
The clipped PPO policy loss is
\[
  \mathcal{L}_{\mathrm{policy}}(\actorparam, \encoderparam)
  =
  - \mathbb{E}_t \left[
  \min\!\left(
  \rho_t(\actorparam)\hat{A}_t,
  \mathrm{clip}(\rho_t(\actorparam), 1-\epsilon, 1+\epsilon)\hat{A}_t
  \right)
  \right].
\]
The value loss is
\[
  \mathcal{L}_{\mathrm{value}}(\valueparam, \encoderparam)
  =
  \mathbb{E}_t \left[
  \left(\modelvalue(\tri_t) - \hat{R}_t\right)^2
  \right],
\]
where $\hat{R}_t$ is the empirical return.
The negative entropy loss is
\[
  \mathcal{L}_{\mathrm{entropy}}(\actorparam, \encoderparam)
  =
  \mathbb{E}_t\left[
  \sum_{Z\in\flipspace{\tri_t}}
  \modelactor(Z\mid\tri_t)
  \log \modelactor(Z\mid\tri_t)
  \right].
\]
The parameters $(\actorparam,\valueparam,\encoderparam)$ are updated by minimizing
\[
  \mathcal{L}
  =
  \mathcal{L}_{\mathrm{policy}}(\actorparam, \encoderparam)
  +
  \alpha_{\mathrm{val}} \mathcal{L}_{\mathrm{value}}(\valueparam, \encoderparam)
  +
  \alpha_{\mathrm{ent}}\mathcal{L}_{\mathrm{entropy}}(\actorparam, \encoderparam),
\]
where $\alpha_{\mathrm{val}}$ is the value coefficient and $\alpha_{\mathrm{ent}}$ is the entropy coefficient, whose value is provided in Appendix~\ref{app:hyperparameters}

\subsection{Model and Training Hyperparameters}
\label{app:hyperparameters}

Table~\ref{tab:model-training-hparams} reports the architecture and PPO settings used in the triangulation optimization experiments. We largely follow the conventional hyperparameters used in~\cite{schulman2017proximal,satorras2021n,ebli2020simplicial}.

\begin{table*}[ht]
\centering
\caption{
  Model and PPO hyperparameters used for \algoname and \textsc{NLS}.
  \textsc{NLS} uses the same encoder and PPO setup as \algoname, but replaces direct feasible-flip scoring with pooling followed by an MLP acceptance head.
}
\label{tab:model-training-hparams}
\small
\setlength{\tabcolsep}{6pt}
\begin{tabular}{lll}
\toprule
Group & Setting & Value \\
\midrule
\multirow{6}{*}{Architecture}
 & EGNN encoder depth & $3$ message-passing layers \\
 & EGNN hidden size & $64$ \\
 & Action-scoring propagation depth & $2$ layers \\
 & Chebyshev order & $3$ \\
 & Value head & $3$-layer MLP \\
 & Actor head & Linear \\
 & Actor and value MLP width & $64$ \\
\midrule
\multirow{9}{*}{PPO}
 & Rollout length & $50$ flips \\
 & Parallel training states & $128$ for $3$D, $512$ for $4$D \\
 & PPO epochs per update & $1$ \\
 & PPO Iterations & $2000$ \\
 & Learning rate & $10^{-4}$ \\
 & Clip parameter & $0.1$ \\
 & Discount factor & $0.99$ \\
 & GAE parameter & $0.95$ \\
 & Value coefficient & $0.5$ \\
 & Entropy coefficient & $0.001$ \\
\midrule
\multirow{3}{*}{\textsc{NLS}}
 & Encoder & same EGNN encoder as \algoname \\
 & Policy head & pooling followed by a $3$-layer MLP of width $64$ \\
 & Training & same PPO settings as \algoname \\
\bottomrule
\end{tabular}
\end{table*}

\FloatBarrier

\section{Experimental Details}
\label{app:experimental-details}

\subsection{Synthetic Polytope Data}
\label{app:synthetic-data}

Algorithm~\ref{alg:datagen} gives the synthetic data-generation procedure used in the triangulation optimization experiments. The purpose of the isomorphism check is to avoid filling the train or test set with many coordinate realizations of the same combinatorial type.
For each accepted training polytope, we use \textsc{TOPCOM} to enumerate up to $2000$ triangulations. These triangulations form the seed set $\hat{V}_{\poly}$ used as initial states during training. If fewer than $2000$ triangulations are available, we keep all enumerated triangulations.
The evaluation dataset follows the same generation procedure as described in Algorithm~\ref{alg:datagen}

\begin{algorithm}[ht]
  \caption{$d$-Polytope Data Generation} \label{alg:datagen}
  \begin{algorithmic}
    \STATE {\bfseries Initialize:} number of sampled points $n$, dataset $\mathcal{D} \gets \emptyset$, target dataset size $K$
    \REPEAT
      \STATE Sample $v_1, v_2, \ldots, v_n \overset{\mathrm{i.i.d.}}{\sim} \mathcal{N}(0, I_d)$
      \STATE $\poly \gets \operatorname{ConvHull}(\{v_1, v_2, \ldots, v_n\})$
      \STATE Set $\operatorname{exist} \gets 0$
      \FOR{$\poly' \in \mathcal{D}$}
        \IF{$\poly'$ is combinatorially isomorphic to $\poly$}
          \STATE $\operatorname{exist} \gets 1$
        \ENDIF
      \ENDFOR
      \IF{$\operatorname{exist} = 0$}
        \STATE $\mathcal{D} \gets \mathcal{D} \cup \{\poly\}$
      \ENDIF
    \UNTIL{$|\mathcal{D}| = K$}
  \end{algorithmic}
\end{algorithm}

\subsection{Baseline Algorithm Details}
\label{app:baseline-details}

All search baselines operate on the same feasible flips enumerated by the geometric routines. 
\textsc{Greedy} chooses the flip with the largest immediate objective improvement. 
\textsc{DFS} performs depth-first search while prioritizing unvisited neighbors with better immediate objective values. 
Best First Search (\textsc{BeFS}) is the best-first frontier baseline reported in the optimization tables. 
\textsc{Simulated Annealing} samples flips according to a probabilistic rule based on immediate objective improvement.
\textsc{NLS} adapts the acceptance-only NeuroLS controller to triangulation optimization. The proposal distribution is the same as \textsc{Simulated Annealing}, while the learned controller decides whether to accept the proposed move. Its encoder architecture matches \algoname, and its policy head follows NeuroLS by pooling the encoded state and applying an MLP.

The expanded baseline definitions are given in the per-method paragraphs below. 
All search methods start from the same initial triangulations as \algoname{} and receive a $500$-flip budget. 
\textsc{TOPCOM} is used for exact enumeration in $3$D and as a $10^9$-enumeration external baseline in $4$D.
The nearby-FRST baselines use the same initial regular triangulations and the same $50$-flip budget as \algoname.

\paragraph{\textsc{Greedy}.}
At each state $\tri$, \textsc{Greedy} enumerates the feasible flips $\flipspace{\tri}$.
It evaluates the resulting objective $f(\flip{\tri}{\circuit})$ for every $\circuit \in \flipspace{\tri}$ and applies the flip with the largest immediate decrease.
If every feasible flip strictly increases $f$, \textsc{Greedy} applies the least-bad flip and continues.
There is no memory and no backtracking, so the trajectory is determined entirely by the local objective landscape.

\paragraph{\textsc{DFS}.}
At each state, \textsc{DFS} expands the current triangulation by enumerating $\flipspace{\tri}$.
It sorts the resulting neighbors by their immediate objective value and pushes the unvisited ones onto a stack with the most promising on top.
The next move pops the top of the stack and applies the corresponding flip.
When a branch has no unvisited neighbor with an improving objective value, \textsc{DFS} backtracks via the stack and continues from the previous unfinished branch.

\paragraph{\textsc{BeFS}.}
At each state, \textsc{BeFS} maintains a global priority queue of all discovered triangulations, keyed by their objective value.
It pops the lowest-objective state and enumerates its feasible flips.
The unvisited neighbors are inserted into the queue with their values.
The next pop expands whichever discovered state currently has the lowest value, even if it is not adjacent to the most recent one.

\paragraph{\textsc{Simulated Annealing}.}
At each state, \textsc{Simulated Annealing} samples a feasible flip $\circuit \sim \mathrm{Uniform}(\flipspace{\tri})$.
It computes the objective change $\Delta = f(\flip{\tri}{\circuit}) - f(\tri)$ and accepts the flip with probability $\min(1, \exp(-\Delta / T_t))$.
The temperature $T_t$ decays over the budget, so non-improving moves are common early in the run and rare near the end.

\paragraph{\textsc{NLS}.}
At each state, \textsc{NLS} samples a candidate flip from the same proposal distribution as \textsc{Simulated Annealing}.
A learned policy then decides whether to accept it.
The shared EGNN encoder embeds the current triangulation, and the embeddings are pooled into a global state representation.
A $3$-layer MLP then outputs the acceptance probability.
The proposed flip is applied with this probability and otherwise rejected, in which case the state remains $\tri$. The implementation is adapted from the official repo : \hyperlink{repo}{https://github.com/jokofa/NeuroLS/tree/master}. 

\paragraph{\textsc{TOPCOM}.}
At each state, \textsc{TOPCOM} performs breadth-first traversal of the flip graph by enumerating the feasible flips and adding the resulting triangulations to the queue.
It tracks visited states to avoid cycles.
In $3$D, this enumerates the entire flip graph and gives the exact optimum used as the gap reference.
In $4$D, the search is capped at $100$M expansions and is reported only as an external sanity check.
In nearby-FRST discovery, \textsc{TOPCOM} runs the same BFS from the initial regular triangulation and stops as soon as an FRST is reached. The implementation is from \hyperlink{repo}{https://github.com/passagemath/upstream-topcom}. 

\paragraph{\textsc{Random Walk}.}
At each state, \textsc{Random Walk} enumerates the feasible flips and applies one chosen uniformly at random.
There is no objective awareness and no memory across steps, so the trajectory is determined entirely by the local geometry of the flip graph.

\paragraph{\textsc{FRST Fast Sampler}.}
At each iteration, \textsc{FRST Fast Sampler} (the \textsc{CYTools} \texttt{random\_triangulations\_fast\_generator}) draws a random height vector $h \in \mathbb{R}^{|\vpoly(\poly)|}$.
It constructs the regular triangulation induced by $h$ and asks \textsc{CYTools} whether the result is an FRST.
New iterations are drawn until the retry stopping rule fires after $50$ consecutive iterations without a new FRST. The implementation is directly from \hyperlink{repo}{https://cy.tools/}.

\paragraph{\textsc{MCMC}.}
At each iteration, \textsc{MCMC} (the \textsc{CYTools} \texttt{random\_triangulations\_fair\_generator}) advances a Markov chain over regular triangulations whose stationary distribution approximates the uniform distribution over FRSTs.
The chain proposes a perturbation of the current height vector and accepts it via the \textsc{CYTools} fairness rule.
Sampling stops under the same $50$-retry rule as \textsc{FRST Fast Sampler}. The implementation is directly from \hyperlink{repo}{https://cy.tools/}.

\textsc{NLS} is the feasible NLS Acceptance variant for this action space. The full NLS design includes policy heads for choosing local-search components that do not map directly to feasible bistellar flips in triangulation optimization. We therefore keep the \textsc{Simulated Annealing} proposal mechanism and learn only whether to accept the proposed move. This makes the comparison test learned acceptance against the direct flip-ranking policy used by \algoname.

For sparse-reward nearby-FRST discovery, \textsc{Greedy}, \textsc{DFS}, \textsc{BeFS}, and \textsc{Simulated Annealing} are not used. Their rules require one-step objective differences, while nearby-FRST discovery gives a binary success signal.

\subsection{Sparse-Reward FRST Details}
\label{app:frst-details}

For nearby-FRST discovery, \algoname~is trained on $200$ reflexive polytopes with $h^{1,1}=3$. Evaluation uses unseen reflexive polytopes with $h^{1,1}\in\{4,6,8,10,12,14,16\}$. For each Hodge number, we sample $20$ polytopes and $50$ random regular triangulations per polytope, giving $1000$ initial states. Each method receives a $50$-flip budget from each initial triangulation. The baselines are \textsc{Random Walk}, which samples uniformly from feasible flips, and \textsc{TOPCOM}, which performs breadth-first flip-graph traversal from the same initial state. Local-improvement baselines are omitted because the reward is binary.

For CY threefold sampling, each \algoname~iteration samples a random height vector in $\mathbb{R}^{|\vpoly(\poly)|}$, constructs the associated regular triangulation, and searches for a nearby FRST with the trained policy. We compare against the \textsc{CYTools} \texttt{random\_triangulations\_fast\_generator} and \texttt{random\_triangulations\_fair\_generator}, reported as \textsc{FRST Fast Sampler} and \textsc{MCMC}. Sampling is evaluated on $10$ reflexive polytopes for each $h^{1,1}\in\{4,8,12,16\}$ and stops for a polytope after $50$ consecutive retries without finding a new FRST. Figure~\ref{fig:cy_class_fraction} reports the final fraction of known CY threefold classes discovered for those evaluated polytopes; the numerator is the number discovered by the method and the denominator is the known total for that setting.

\FloatBarrier

\section{Complementary Optimization Results}
\label{app:additional-experiments}

Tables~\ref{tab:results-3d} and~\ref{tab:results-4d} report the numerical values for the $500$-flip optimization comparison. 
Each entry is the mean relative gap with standard error. 
The average row aggregates per-instance gaps over all listed objectives and evaluation sizes.
Lower is better.

\begin{table*}[ht]
\centering
\caption{
  Relative gap (\%, $\downarrow$) after $500$ bistellar flips on unseen simplicial polytopes in $3$D.
  The reference is the exact optimum found by exhaustive enumeration.
  Bold and underlined entries mark the best and second-best $500$-flip search methods.
}
\label{tab:results-3d}
\scriptsize
\setlength{\tabcolsep}{3pt}
\resizebox{\textwidth}{!}{%
\begin{tabular}{ccl cccccc}
\toprule
\multicolumn{3}{c}{Evaluation Setup} & \multicolumn{6}{c}{Relative Gap @ 500 (\%, $\downarrow$)} \\
\cmidrule(lr){1-3}
\cmidrule(lr){4-9}
$d$ & $|\vpoly|$ & Objective & \textsc{Greedy} & \textsc{DFS} & \textsc{BeFS} & \textsc{SA} & \textsc{NLS} & \textbf{Ours} \\
\midrule
\multirow{6}{*}{3}
 & \multirow{3}{*}{13}
   & Min Simplices & 23.11{\tiny$\pm$2.69} & 6.41{\tiny$\pm$0.95} & \underline{2.14{\tiny$\pm$0.62}} & 6.64{\tiny$\pm$0.86} & 13.00{\tiny$\pm$2.35} & \textbf{0.13{\tiny$\pm$0.13}} \\
 & & Min Diameter & 14.30{\tiny$\pm$2.20} & 11.50{\tiny$\pm$2.12} & 16.70{\tiny$\pm$2.19} & 12.30{\tiny$\pm$2.11} & \underline{3.90{\tiny$\pm$1.27}} & \textbf{2.00{\tiny$\pm$0.96}} \\
 & & Min Weight & 10.56{\tiny$\pm$0.85} & 7.18{\tiny$\pm$0.69} & \underline{3.54{\tiny$\pm$0.58}} & 11.42{\tiny$\pm$0.92} & 8.12{\tiny$\pm$0.99} & \textbf{0.15{\tiny$\pm$0.06}} \\
\cmidrule(lr){2-9}
 & \multirow{3}{*}{14}
   & Min Simplices & 23.51{\tiny$\pm$2.50} & 12.07{\tiny$\pm$2.05} & \underline{6.17{\tiny$\pm$1.30}} & 12.60{\tiny$\pm$1.51} & 14.25{\tiny$\pm$2.18} & \textbf{0.00{\tiny$\pm$0.00}} \\
 & & Min Diameter & 15.40{\tiny$\pm$1.82} & 16.50{\tiny$\pm$1.10} & 20.60{\tiny$\pm$1.71} & 18.20{\tiny$\pm$1.12} & \underline{6.10{\tiny$\pm$1.32}} & \textbf{0.90{\tiny$\pm$0.63}} \\
 & & Min Weight & 11.78{\tiny$\pm$0.89} & 8.87{\tiny$\pm$0.79} & \underline{6.21{\tiny$\pm$0.76}} & 11.87{\tiny$\pm$0.98} & 8.43{\tiny$\pm$0.92} & \textbf{0.15{\tiny$\pm$0.08}} \\
\midrule
\multicolumn{3}{c}{Average} & 16.44{\tiny$\pm$0.85} & 10.42{\tiny$\pm$0.61} & 9.23{\tiny$\pm$0.67} & 12.17{\tiny$\pm$0.57} & \underline{8.97{\tiny$\pm$0.69}} & \textbf{0.56{\tiny$\pm$0.20}} \\
\bottomrule
\end{tabular}%
}
\end{table*}

\begin{table*}[ht]
\centering
\caption{
  Relative gap (\%, $\downarrow$) after $500$ bistellar flips on unseen simplicial polytopes in $4$D.
  The \textsc{TOPCOM} result is obtained by $10^9$ enumeration, and the reference is the best value found between this value and the $500$-flip search methods.
  Bold and underlined entries mark the best and second-best $500$-flip search methods.
}
\label{tab:results-4d}
\scriptsize
\setlength{\tabcolsep}{3pt}
\resizebox{\textwidth}{!}{%
\begin{tabular}{ccl ccccccc}
\toprule
\multicolumn{3}{c}{Evaluation Setup} & \multicolumn{7}{c}{Relative Gap @ 500 (\%, $\downarrow$)} \\
\cmidrule(lr){1-3}
\cmidrule(lr){4-10}
$d$ & $|\vpoly|$ & Objective & $10^9$ \textsc{TOPCOM} & \textsc{Greedy} & \textsc{DFS} & \textsc{BeFS} & \textsc{SA} & \textsc{NLS} & \textbf{Ours} \\
\midrule
\multirow{6}{*}{4}
 & \multirow{3}{*}{13}
   & Min Simplices & 2.41{\tiny$\pm$0.94} & 17.25{\tiny$\pm$4.68} & 20.01{\tiny$\pm$6.30} & 19.24{\tiny$\pm$5.17} & \underline{6.93{\tiny$\pm$2.14}} & 11.29{\tiny$\pm$3.12} & \textbf{0.00{\tiny$\pm$0.00}} \\
 & & Min Diameter & 1.00{\tiny$\pm$0.97} & 20.00{\tiny$\pm$0.00} & 20.00{\tiny$\pm$0.00} & 20.00{\tiny$\pm$0.00} & 19.00{\tiny$\pm$0.97} & \underline{13.00{\tiny$\pm$2.13}} & \textbf{1.00{\tiny$\pm$0.97}} \\
 & & Min Weight & 3.02{\tiny$\pm$0.59} & 4.82{\tiny$\pm$1.23} & 5.28{\tiny$\pm$1.45} & 4.33{\tiny$\pm$1.18} & 2.93{\tiny$\pm$0.78} & \underline{2.34{\tiny$\pm$0.65}} & \textbf{0.37{\tiny$\pm$0.16}} \\
\cmidrule(lr){2-10}
 & \multirow{3}{*}{14}
   & Min Simplices & 14.96{\tiny$\pm$2.57} & 55.82{\tiny$\pm$7.64} & 23.16{\tiny$\pm$4.51} & 23.09{\tiny$\pm$4.98} & 14.69{\tiny$\pm$3.30} & \underline{11.42{\tiny$\pm$2.99}} & \textbf{0.37{\tiny$\pm$0.36}} \\
 & & Min Diameter & 6.00{\tiny$\pm$2.05} & 30.83{\tiny$\pm$2.39} & 25.33{\tiny$\pm$3.17} & 22.67{\tiny$\pm$3.27} & 18.67{\tiny$\pm$2.64} & \underline{16.83{\tiny$\pm$2.55}} & \textbf{5.00{\tiny$\pm$1.94}} \\
 & & Min Weight & 4.24{\tiny$\pm$0.62} & 12.84{\tiny$\pm$1.71} & 5.17{\tiny$\pm$1.07} & 5.00{\tiny$\pm$1.03} & \underline{1.86{\tiny$\pm$0.64}} & 1.98{\tiny$\pm$0.60} & \textbf{0.54{\tiny$\pm$0.19}} \\
\midrule
\multicolumn{3}{c}{Average} & 5.27{\tiny$\pm$0.74} & 23.59{\tiny$\pm$2.18} & 16.49{\tiny$\pm$1.61} & 15.72{\tiny$\pm$1.52} & 10.68{\tiny$\pm$1.05} & \underline{9.48{\tiny$\pm$1.05}} & \textbf{1.21{\tiny$\pm$0.40}} \\
\bottomrule
\end{tabular}%
}
\end{table*}
\vfill 

% \clearpage
% \input{checklist.tex}

\end{document}